\PassOptionsToPackage{pdfpagelabels=false}{hyperref} 
\documentclass[english,a4paper,11pt]{article}

\usepackage{fullpage}

\usepackage{amsxtra, amsfonts, amssymb, amstext}
\usepackage{amsthm}
\usepackage{booktabs}
\usepackage{longtable}
\usepackage{nicefrac}
\usepackage{xspace}
\usepackage[noadjust]{cite}
\usepackage{url}
\usepackage{graphics,color}
\usepackage[colorlinks]{hyperref}
\definecolor{linkblue}{rgb}{0.1,0.1,0.8}
\hypersetup{colorlinks=true,linkcolor=linkblue,filecolor=linkblue,urlcolor=linkblue,citecolor=linkblue}
\usepackage[algo2e,ruled,vlined,linesnumbered]{algorithm2e}
\newcommand{\assign}{\leftarrow}
\usepackage{graphicx}
\usepackage{wrapfig}

\newtheorem{theorem}{Theorem}
\newtheorem{lemma}[theorem]{Lemma}

\newcommand{\ignore}[1]{}

\allowdisplaybreaks[1]

\newcommand{\N}{\mathbb{N}}
\newcommand{\R}{\mathbb{R}}

\renewcommand{\epsilon}{\varepsilon}

\newcommand{\A}{\mathcal{A}}
\newcommand{\Pe}{\mathcal{P}}

\DeclareMathOperator{\mut}{mut}
\DeclareMathOperator{\cross}{cross}

\newcommand{\E}{\mathbb{E}}
\renewcommand{\Pr}{\mathbb{P}}

\DeclareMathOperator{\shift}{0\rightarrow 1}

\DeclareMathOperator{\Bin}{Bin}

\newcommand{\onemax}{\textsc{OneMax}\xspace}

\newcommand{\OneMax}{\textsc{OneMax}\xspace}
\newcommand{\OM}{\textsc{Om}\xspace}

\newcommand{\leadingones}{\textsc{LeadingOnes}\xspace}
\newcommand{\LO}{\textsc{Lo}\xspace}

\newcommand{\oea}{$(1 + 1)$~EA\xspace}

\newcommand{\ga}{$(1 + (\lambda,\lambda))$~GA\xspace}
\newcommand{\gaopt}{$(1 + (\lambda,\lambda))$~GA$_\text{mod}$\xspace}

\newcommand{\RLSopt}{RLS$_{\text{opt}}$\xspace}
\newcommand{\oeares}{$(1 + 1)$~EA$_{>0}$\xspace}
\newcommand{\oeashift}{$(1 + 1)$~EA$_{0\rightarrow 1}$\xspace}
\newcommand{\tpo}{Greedy $(2+1)$~GA\xspace}
\newcommand{\tpores}{Greedy $(2+1)$~GA$_{\text{mod}}$\xspace}

\title{Towards a More Practice-Aware Runtime Analysis of Evolutionary Algorithms}

\author{
Eduardo Carvalho Pinto$^1$
\and
Carola Doerr$^2$
}

\date{
$^1$DreamQuark, Paris, France\\
$^2$Sorbonne Universit\'e, CNRS, LIP6, Paris, France\\[2ex]
The report is as of July 2017\\[1.5ex]
{\footnotesize{This report summarizes results obtained during the Master (M2) internship of Eduardo at LIP6. Some of the results have been communicated at EA 2017~\cite{CarvalhoD17} and PPSN 2018~\cite{CarvalhoD18}. In December 2018 we decided to make this report available because it contains some otherwise unpublished results. We did not revise the text though. Some references may therefore be outdated, please carefully check the papers~\cite{CarvalhoD17,CarvalhoD18} for updated references. If you are interested in a specific result and want to know the latest version, please do not hesitate to get in touch.}}}

\begin{document}
\maketitle 

{\sloppy
\begin{abstract}
Theory of evolutionary computation (EC) aims at providing mathematically founded statements about the performance of evolutionary algorithms (EAs). The predominant topic in this research domain is \emph{runtime analysis}, which studies the time it takes a given EA to solve a given optimization problem. Runtime analysis has witnessed significant advances in the last couple of years, allowing us to compute precise runtime estimates for several EAs and several problems.

Runtime analysis is, however (and unfortunately!), often judged by practitioners to be of little relevance for real applications of EAs. Several reasons for this claim exist. We address two of them in this present work: 
(1) EA implementations often differ from their vanilla pseudocode description, which, in turn, typically form the basis for runtime analysis. To close the resulting gap between empirically observed and theoretically derived performance estimates, we therefore suggest to take this discrepancy into account in the mathematical analysis and to adjust, for example, the cost assigned to the evaluation of search points that equal one of their direct parents (provided that this is easy to verify as is the case in almost all standard EAs).
(2) Most runtime analysis results make statements about the expected time to reach an optimal solution (and possibly the distribution of this optimization time) only, thus explicitly or implicitly neglecting the importance of understanding how the function values evolve over time. We suggest to extend runtime statements to \emph{runtime profiles}, covering the expected time needed to reach points of intermediate fitness values.

While certainly other reasons exist to believe that runtime analysis is of limited practical relevance, and while our suggested solutions can only serve as a pointer to a more practice-aware runtime analysis theory, we are confident that our work helps to initiate a more constructive exchange between theoretical and empirically-driven research in EC.

As a more direct consequence, we obtain a result showing that the greedy (2+1) GA of Sudholt [GECCO 2012] outperforms any unary unbiased black-box algorithm on \onemax, thus giving strong evidence that our suggested performance measure has the potential to drastically change the theoreticians' view on long-standing hypotheses in evolutionary computation. 
\end{abstract} 


\section{Introduction}%
\label{sec:introduction}

Evolutionary algorithms (EAs) are bio-inspired black-box optimization heuristics that are successfully applied to a broad range of industrial and academic optimization problems. Their practical relevance has motivated theoreticians to analyze EAs by mathematically means, aiming at providing mathematically founded insights into the working principles of EAs. Unfortunately, we see today a rather big gap between theoretical and practice-driven research in evolutionary computation (EC). Unlike in classical computer science, where a fruitful interplay between mathematically- and empirically-driven research exists, theory of EAs is regularly considered ``useless'' in the more practically-oriented part of the EC research community, cf., e.g., Footnote 4 in~\cite{EibenHM99}, where it is stated that ``the best thing a practitioner of EA's can do is to stay away from theoretical results''. 
One critical reason for such claims is the fact that EAs are particularly useful when the problem at hand is not analyzable by thorough mathematical means; while for problems that do admit such a mathematical approach, problem-tailored algorithms are typically much more powerful than heuristics. This gap is very difficult, if not impossible, to close. At the same time, however, we also see that some of the other reasons for practitioners not to follow too closely what theory of EC can offer could be easily addressed. We suggest in this work two different steps in this direction, i.e., towards a more practice-aware theory of EAs. Our hope is to trigger with this work a more constructive exchange between theoretical and empirical research streams in EC. 

As a side effect of our work, we also obtain some theoretical results that are interesting in their own rights. We summarize our proposed changes and results in Sections~\ref{sec:intro-runtime} and~\ref{sec:intro-profiles}.

\subsection{Implementation-Aware Runtime Analysis}
\label{sec:intro-runtime}
The by-far most relevant performance measure in discrete black-box optimization is the number of fitness evaluations that an algorithm needs until it evaluates for the first time an optimal search point. This \emph{runtime} (also: \emph{optimization time}) measure is often motivated by the fact that in typical applications the function evaluations are the most costly part of the black-box optimizer. Another motivation for this performance measure is the fact that black-box optimizers are often used in situations where the data is subject to privacy concerns, so that the data owner will be able to answer individual function evaluation requests but does not or cannot reveal any other information about its data. Counting function evaluations is a standard complexity measure both in EC but also in classical computer science, where it is studied under the notion of \emph{query complexity}. 

One seemingly negligible difference between the runtime studied in the theory of EAs and that studied in practice- or empirically-motivated work is the fact that in most theory works function evaluations are counted also for iterations in which the offspring equals its parent (or one of its parents in case of recombination). In many EAs this situation occurs frequently. For example, when using standard bit mutation (cf. Section~\ref{sec:implement} for a description of the here-mentioned variation operators, algorithms, and problem classes) with mutation rate $p=1/n$, the probability that the mutated offspring equals its parent is $(1-p)^n \approx 1/e = 0.368$, showing that, for example, the \oea, which only uses standard bit mutation as variation operator, uses a $1/e$ fraction of its function evaluations for offspring that are identical to the parent. In practical implementations one would of course avoid such evaluations, in particular when it is easy to check (as in this case) that the offspring equals its parent. 

A similar situation, in which the offspring is likely to equal one of its parents, occurs when two similar parents are recombined. This is quite frequent in typical $(\mu+\lambda)$~GAs. Furthermore, a biased crossover favoring entries from one of the parents as, for example, used in the \ga proposed in~\cite{DoerrDE15}, has a relatively high probability to reproduce one of the parent solutions even if the Hamming distance of the two parents is large. Since in these cases it is typically quite easy to determine if the recombined offspring equals one of its parents, the offspring would not be evaluated in typical implementations, while in existing theoretical runtime analysis statements we would charge the algorithm one function evaluation for creating this offspring.

These discrepancies between the analyzed and the actually implemented EAs can result in significant gaps between empirically observed and mathematically derived performance estimates. To reduce this gap, we suggest to reflect these observations in the runtime bounds, by analyzing EA variants in which offspring are not evaluated if they equal one of their parents and this equality is easy to determine. In the case of standard bit mutation, two straightforward ways to implement this strategy exist. The first one simply ignores such iterations, thus effectively resampling an offspring until it differs from its parent. This strategy can be efficiently implemented by sampling the mutated offspring from a conditional distribution. 
The second idea is to flip \emph{one} random bit in case no bit is flipped in the mutation step (thus effectively shifting the probability mass to flip zero bits to the probability to flip exactly one bit). In the case of crossover there are many ways to deal with such situations. When crossover is followed by mutation, we suggest to use one of the just-mentioned two solutions to ensure flipping at least one bit in the mutation step if the offspring created by recombination is merely a copy of one of its parents. When crossover is not followed by mutation (as in the \ga), we suggest to simply omit the function evaluation of such offspring. 

We analyze in Sections~\ref{sec:oea-implement} and~\ref{sec:ga-implement} how these suggestions influence the performances of the \oea, the \ga proposed in~\cite{DoerrDE15}, and the \tpo from~\cite{Sudholt12}. While the theoretical bounds are easy to obtain, we observe a quite surprising result. We show in Theorem~\ref{thm:tpo} that our simple modification of the \tpo yields a better expected optimization time on \onemax than any unary unbiased (i.e., mutation-based) black-box algorithm. As far as we know, this shows for the first time that a classical genetic algorithm with crossover can outperform all unary unbiased black-box algorithms on the simple hill climbing problem \onemax. Given the long series of works trying to establish such results (cf. the literature surveys in~\cite{DoerrDE15,Sudholt12}), this shows that our simple suggestion of a more implementation-aware runtime analysis can have substantial impact on the theoretician's view on some of the most fundamental questions in EC. 

We note that previous attempts to establish more meaningful complexity measures in EC exists, most notably in the work of Jansen and Zarges~\cite{JansenZ11}, who propose to profile how much time is spend in each of the steps of an EA relative to the cost of the function evaluation. Jansen and Zarges also briefly discuss the here-proposed resampling variant of the \oea, but conclude from their work that ``the simple cost model [of not counting mutation steps of the \oea in which no bit is flipped] is inadequate''. We do agree that counting function evaluations may not always reflect very well the wall-clock time spend on a problem (in particular if Hash-tables are used to achieve the already evaluated search points or if evaluation can be done in constant time for offspring resembling previously evaluated ones). It is still, as mentioned above, the most relevant cost measure in EC. The research question that we address in this work is how to deal with offspring that equal their parents, and our suggestion is to adopt a more implementation-aware view in runtime analysis. 

\subsection{Runtime Profiles}
\label{sec:intro-profiles}
Our second suggestion concerns the problem that in most theoretical works on EAs for discrete optimization problems, only the expected times to hit an optimal solution for the first time are reported. In realistic environments, we do not know when this is the case, making it equally important to understand how the fitness values \emph{evolve} over time. We therefore introduce in Section~\ref{sec:profiles} the concept of \emph{runtime profiles}, the expected time needed to hit intermediate target values. In the case of \onemax or \leadingones functions of dimension~$n$, these runtime profiles could be the expected times to reach any fitness level $i \in [n]$, while for problems taking less canonical fitness values, other intermediate target values could be used. In short, our suggestion is to state \emph{runtime profiles} of an algorithm rather than only the optimization time. 

As we shall discuss in Section~\ref{sec:profiles}, it is quite interesting to observe how the first hitting times for the different targets evolve. Already for the simple $\leadingones$ benchmark functions we see that the \oea-variants proposed in Section~\ref{sec:oea-implement} are superior to RLS for all target fitness values $i$ that are smaller than some relatively large threshold value $v$, while RLS reaches fitness levels $i \geq v$ faster than the \oea-counterparts. Reporting only the expected optimization time therefore does not do justice to the better performance of the \oea-variants in the earlier parts of the optimization process. 

We are confident that the proposed runtime profiles will be very useful for the design and the analysis of non-static parameter and operator choices in EAs, such as adaptive mutation or crossover rates, selection pressure, etc. This topic, also studied under the notions of \emph{hyper-heuristics}, \emph{meta-heuristics}, etc., has very recently seen increased interest in the theory of EC community~\cite{DoerrD15self,DoerrDY16PPSN,DoerrDK16PPSN,DangL16}. It is a highly relevant topic in empirically-driven research in EC, cf.~\cite{AletiM16,EibenHM99,EibenMSS07,KarafotiasHE15} and references therein.

Our runtime profiles complement the \emph{fixed-budget perspective} proposed in~\cite{JansenZ14}, where statements about the expected fitness values after a fixed number of fitness evaluations are sought. Runtime profiles and fixed-budget perspectives are orthogonal views on the performance of EAs, both aiming at providing more insight into the optimization behavior than what the single runtime measure can offer. Neither of these two measures is entirely new but rather summarize performance statements frequently reported in empirical works. The aim of~\cite{JansenZ14} as well as our own work is to motivate researchers working in the theory of EC to include in their statements these more informative performance guarantees.    

\textbf{Complexity of Implementing Our Suggestions.} We emphasize that from a theoretical point of view all our suggestions are easy to implement. Indeed, all results reported in this work can be easily derived from existing runtime results. But, as our result for the \tpores shows, it may drastically change our view on classical EAs. 

\textbf{Appendix.} Detailed experimental data for the figures reported in the main file as well as some technical statements needed in our proofs are reported in the appendix. 

\section{Implementation-Aware Runtime Analysis}
\label{sec:implement}  

Standard implementations of EAs often differ from their vanilla pseudocode descriptions. These ``tweaks'' are either strictly needed to make an algorithmic idea implementable 
or just ``nice to have'', in order to speed up an algorithm without changing its performance. In this section, we describe some of such typical differences, and analyze their impact on standard runtime results in EC. 

\textbf{Scope:} In the following, we consider the maximization of single-objective pseudo-Boolean functions $f:\{0,1\}^n \rightarrow \R$, but all our suggestion can be applied to other search and objective spaces. 

\textbf{Notation:} We use $n$ to denote the problem dimension. We abbreviate $[n]:=\{1,2,\ldots,n\}$ and $[0..n]:=\{0\} \cup [n]$. By $\ln$ we denote the natural logarithm to base $e:=\exp(1)$. For every positive integer $d$ we denote by $H_d$ the $d$-th harmonic number $H_d:=\sum_{i=1}^d{1/i} = \ln(d) + \gamma + O(1/d)$ with $\gamma \approx 0.5772156649$ being the Euler–Mascheroni constant.

\subsection{The (1+1) EA}
\label{sec:oea-implement}

One of the best-studied EAs in the theory of EC is the \oea. It has a simple structure and is often used as a showcase to understand the role of global mutation in combination with elitist selection. 

The \oea works as follows. It has a very restricted memory (``\emph{population}''), keeping only the best so-far solution candidate in the memory (and the most recent one in case several search points of current-best function value have been evaluated). In the \emph{mutation step}, an \emph{offspring} $y$ is sampled from this current-best solution $x$ by changing each bit in $x$ with some \emph{mutation probability} $p \in (0,1)$, independently of all other bits. In the \emph{selection step}, the \emph{parent} $x$ is replaced by its offspring $y$ if and only if the \emph{fitness} of $y$ is at least as good as the one of $x$; i.e., in the case of maximization, if and only if $f(y) \geq f(x)$. 

When implementing the \oea, it would be inconvenient and time-consuming to sample in each iteration and for each bit the random Bernoulli variable describing whether or not the corresponding bit should be flipped. A common way to implement the standard bit mutation is to sample in the beginning of the mutation step a random variable $\ell$ from the binomial distribution $\Bin(n,p)$ with $n$ trials and success probability $p$, i.e., $\Pr[\ell=k] = \binom{n}{k} p^k (1-p)^{n-k}$ for all $k \in [0..n]$. Once $\ell$ is sampled, $\ell$ different (i.e., without replacement) positions $i_1,\ldots,i_\ell \in [n]$ are chosen uniformly at random and $y$ is created from $x$ by copying $y_i := x_i$ for $i \in [n] \setminus \{ i_1,\ldots,i_\ell\}$ and setting $y_i := 1-x_i$ for $i \in \{ i_1,\ldots,i_\ell\}$. It is not difficult to verify that this implementation is identical to the one with $n$ independent Bernoulli trials.

This way, the \oea can be stated in the form of Algorithm~\ref{alg:oea}. Note here that no termination criterion is specified. This is justified by the fact that runtime analysis studies the expected time this algorithm needs to find an optimal solution. In real implementations, of course, one has to specify a termination criterion, which can be, for example, an upper bound on the number of iterations, on CPU time, or the number of iterations in which no improvement has been observed, but also the first point in time a certain fitness-value has been reached, etc. 

 \begin{algorithm2e}[t]%
	\textbf{Initialization:} 
	Sample $x \in \{0,1\}^{n}$ uniformly at random and compute $f(x)$\;
  \textbf{Optimization:}
	\For{$t=1,2,3,\ldots$}{
		Sample $\ell \sim \Bin(n,p)$\label{line:ell}\;
		$y \assign \mut_{\ell}(x)$\;
		evaluate $f(y)$\label{line:oeaeval}\;
		\lIf{$f(y)\geq f(x)$}{$x \assign y$}	
}
\caption{The well-known \oea with mutation probability $p \in (0,1)$ for the maximization of a pseudo-Boolean function $f:\{0,1\}^n \rightarrow \R$}
\label{alg:oea}
\end{algorithm2e}

\begin{algorithm2e}%
	\textbf{Input:} $x \in \{0,1\}^n$, $\ell \in \N$\;
		\label{line:elloea}Select $\ell$ different positions $i_1,\ldots,i_{\ell} \in [n]$ u.a.r.\;
	  $y \assign x$\;
		\lFor{$j=1,...,\ell$}{$y_{i_j}\assign 1-x_{i_j}$}
\caption{$\mut_{\ell}$ chooses $\ell$ different positions and flips the entries in these positions.}
\label{alg:mut}
\end{algorithm2e}

Implementing the Bernoulli trials of the standard bit mutation as above did not affect the performance of the \oea. But there is another tweak that can speed up the algorithm significantly. The idea is simple and therefore found in most standard implementations of the \oea. The probability that $\ell$ in line~\ref{line:ell} of Algorithm~\ref{alg:oea} equals zero is $(1-p)^n$. For the often recommended choice $p=1/n$ this results in a $1/e$ fraction of iterations in which no bit is flipped at all. In the vanilla description of the \oea above, we would still count one function evaluation for any of these iterations (cf. line~\ref{line:oeaeval}). In practice, however, we would rather ignore such iterations by (1) either re-sampling $\ell$ until we sample a value $\ell>0$ (this is the case if we simply skip line~\ref{line:oeaeval} in Algorithm~\ref{alg:oea} whenever $\ell=0$) or (2) by using $\ell=1$ if $\ell=0$ is sampled. We refer to the 
\oea using the first option as \oeares (``resampling EA'', Algorithm~\ref{alg:oeares}), and we call the algorithm using the second option the \oeashift (``shift EA'', Algorithm~\ref{alg:oeashift}). The \oeares can efficiently be implemented by sampling $\ell$ from the conditional distribution $\Bin_{>0}(n,p)$, which assigns to each $k$ a probability of $\Pr[\ell=k|\ell>0] = \binom{n}{k} p^k (1-p)^{n-k}/(1-(1-p)^n)$.

 \begin{algorithm2e}%
	\textbf{Initialization:} 
	Sample $x \in \{0,1\}^{n}$ uniformly at random and compute $f(x)$\;
  \textbf{Optimization:}
	\For{$t=1,2,3,\ldots$}{
		\label{line:ellres}Sample $\ell \sim \Bin_{>0}(n,p)$\;
		$y \assign \mut_{\ell}(x)$\;
		evaluate $f(y)$\;
		\lIf{$f(y)\geq f(x)$}{$x \assign y$}	
}
\caption{The \oeares with mutation probability $p \in (0,1)$ for the maximization of a pseudo-Boolean function $f:\{0,1\}^n \rightarrow \R$}
\label{alg:oeares}
\end{algorithm2e}

 \begin{algorithm2e}%
	\textbf{Initialization:} 
	Sample $x \in \{0,1\}^{n}$ uniformly at random and compute $f(x)$\;
  \textbf{Optimization:}
	\For{$t=1,2,3,\ldots$}{
		\label{line:ellshift}Sample $\ell \sim \Bin(n,p)$\;
		\lIf{$\ell=0$}{$\ell \assign 1$}
		$y \assign \mut_{\ell}(x)$\;
		evaluate $f(y)$\;
		\lIf{$f(y)\geq f(x)$}{$x \assign y$}	
}
\caption{The \oeashift with standard mutation probability $p \in (0,1)$ for the maximization of a pseudo-Boolean function $f:\{0,1\}^n \rightarrow \R$}
\label{alg:oeashift}
\end{algorithm2e}

In the next two subsections, we discuss how these strategies change the expected optimization time of the \oea on \onemax (Sec.~\ref{sec:implementOM}) 
and on \leadingones (Sec.~\ref{sec:implementLO}), respectively. 
We refer the interested reader to Section~\ref{app:benchmark} in the appendix for a discussion of the regarded benchmark problems.

\subsubsection{Performance of the (1+1) EA Variants on OneMax}
\label{sec:implementOM}

We start our investigations by regarding \onemax, the best-understood benchmark problem in runtime analysis. \onemax is the function that returns as function value the number of ones in a string, i.e., $\OM(x):=\sum_{i=1}^n{x_i}$. We compare the performance of the \oea, the \oeares, and the \oeashift. We also add to our experiments a comparison with Randomized Local Search (RLS), the standard first-ascent hill climber which samples in each iteration a random neighbor at Hamming distance 1 from the current-best solution. We obtain RLS from Algorithm~\ref{alg:oea} by replacing line~\ref{line:ell} with ``Set $\ell =1$;''. 

From a mathematical point of view, the suggested changes do not cause much trouble. In fact, the runtime behavior of the \oea on \onemax (cf.~\cite{HwangPRTC14} and references mentioned therein) and, more generally, linear functions $f:\{0,1\}^n \rightarrow \R, x \mapsto \sum_{i=1}^n{w_i x_i}$ with $w_i \in \R$~\cite{Witt13j} is quite well understood. The following result easily follow from standard fitness level arguments and illustrates how existing runtime results for the \oea can easily be adapted to the \oeares and the \oeashift. 

\begin{theorem}
\label{thm:oeaOM}
The expected optimization time of the \oeares with mutation probability $p \in (0,1)$ for \onemax is at most 
$\frac{ 1  - \left(1 - p \right) ^ n}{p\left( 1 - p \right) ^ {n-1}} H_n$ 
and for the \oeashift it is at most $\frac{1}{(1-p)^{n-1} (np + 1 - p) } nH_n$. 
\end{theorem}
\begin{proof}[Proof of the upper bound for the \oeares on \onemax in Theorem~\ref{thm:oeaOM}]
We do a simple fitness level argument, using the canonical fitness level partition $L_0,L_1,\ldots,L_n$ with $L_i:=\{ x \in \{0,1\}^n \mid \OM(x)=i\}$. 

For $i \in \{0,\dots, n-1\}$ let $p_i$ be the probability to leave fitness level $L_i$ in one iteration of the \oeares with mutation probability $p$ when starting in some $x \in L_i$. Since the fitness $L_i$ is left if one of the $(n-i)$ 0-bits and no other bit is flipped, we can bound $p_i$ from below as
\[ 
p_i \ge \frac{\binom{n-i}{1} p \left( 1 - p \right) ^ {n-1}}{1 - \left(1 - p \right) ^ n},
 \]
using that the probability to sample $\ell=1$ in line~\ref{line:ellres} of Algorithm~\ref{alg:oeares} is $p \left( 1 - p \right) ^ {n-1}/(1 - \left(1 - p \right) ^ n)$.
 Summing up the expected waiting times $1/p_i$, we thus get
\begin{align*}
\mathbb{E}[T_{>0,p}(\OM)] 
&\le \sum\limits_{i = 0}^{n-1} \frac{1}{p_i} 
\le \sum\limits_{i=0}^{n-1} \left( 1  - \left(1 - p \right) ^ n \right) \frac{1}{n-i} \frac{1}{ p \left( 1 - p \right) ^ {n-1}} 
\le \frac{\left( 1  - \left(1 - p \right) ^ n \right)}{p(1-p)^{n-1}} H_n. 
\end{align*}
 \end{proof}

The proof for the \oeashift is very similar. 
\begin{proof}[Proof of the upper bound for the \oeashift on \onemax in Theorem~\ref{thm:oeaOM}]
We use the same fitness level partition as for the \oeares. Letting $p_i$ be the probability to leave fitness level $L_i$ in one iteration of the \oeashift with mutation probability $p$ when starting in some $x \in L_i$, we bound  
\begin{align*}
p_i 
&\ge \binom{n-i}{1} p \left( 1 - p \right) ^ {n-1} + \frac{n-i}{n} (1-p)^n \\
& \ge (1-p)^{n-1} \left( (n-i)p + \frac{n-i}{n}(1-p) \right) \\
& \ge (1-p)^{n-1}(np + 1 - p) \frac{n-i}{n}. 
\end{align*}
Summing up expected waiting times $1/p_i$ yields
\begin{align*}
\mathbb{E}[T_{\shift,p}(\OM)] 
&\le \frac{1}{(1-p)^{n-1} (np + 1 - p) } \sum\limits_{i=0}^{n-1} \frac{n}{n-i} 
\le \frac{1}{(1-p)^{n-1} (np + 1 - p) } nH_n.
\end{align*}
\end{proof}

For the most often recommended choice $p=1/n$, the expected optimization time of the \oea equals
$e n \ln(n) -  1.8925... n + \tfrac{e}{2} \ln(n) + 0.5978... + O(\log n / n)$~\cite{HwangPRTC14}, while the bounds in Theorem~\ref{thm:oeaOM} evaluate to
$(e-1+\tfrac{1}{n})nH_n \approx 1.718 n \ln(n) +0.918 n + \ln(n) + O(1)$ for the \oeares and 
$\frac{e}{2-1/n}nH_n \approx 1.3591 n \ln(n)+0.7845 n + O(1)$ for the \oeashift. 
For this choice of $p$, the bound for the \oeares is tight, as the following theorem shows. 
\begin{theorem}
\label{thm:oeaOMlower}
Let $p = O(n^{-2/3 - \epsilon})$. Then the expected time of the \oeares with mutation probability $p$ on \onemax 
bounded from below by
\[ (1 - o(1))\frac{1 - (1-p)^n}{p(1-p)^n} \min\{\ln n, \ln (1/(p^3n^2)) \}. \]

In particular, for any constant $c>0$, the expected runtime of the \oeares with mutation probability $p=c/n$ on \onemax is at least $(1 - o(1))\frac{e^c - 1}{c} n \ln n$.
\end{theorem}

Our proof of Theorem~\ref{thm:oeaOMlower} follows very closely that of Theorem~6.5 in~\cite{Witt13j}. Note that Witt actually proves a lower bound for what he calls ``arbitrary mutation-based EAs'' or ``(1+1) EA$_{\mu}$'', i.e., EAs using as starting point the best out of $\mu$ search points drawn independently and uniformly at random and then using standard bit mutation as only means of variation. Our statement also holds for this larger class of algorithms, i.e., the class of all (1+1) EA$_{\mu,>0}$ not evaluating offspring that equal their immediate parent. 

A useful tool in his analysis is the following drift theorem, also proven in~\cite{Witt13j}. 
\begin{theorem}[Multiplicative Drift Lower Bound from~\cite{Witt13j}]
\label{thm:multidriftlower}
Let $S \subseteq \mathbb{R}$ be a finite set of positive numbers with minimum 1. Let $\{X^{(t)}\}_{t \ge 0}$ be a sequence of random variables over $S$, where $X^{(t+1)} \le X^{(t)}$ for any $t \ge 0$ and let $s_{\min} > 0$. Let $T$ be the random variable that gives the first in point time $t \ge 0$ for which $X^{(t)} \le s_{\min}$. If there exist positive reals $\beta,\delta \le 1$ such that, for all $s > s_{\min}$ and all $t \ge 0$ with $\mathbb{P}(X^{(t)} = s) > 0$,
\begin{enumerate}
\item $\mathbb{E}[X^{(t)} - X^{(t+1)} | X^{(t)} = s] \le \delta s$,
\item $\mathbb{P}(X^{(t)} - X^{(t+1)} \ge \beta s | X^{(t)} = s) \le \beta \delta/ \ln s$,
\end{enumerate}
then for all $s_0 \in S$ with $\mathbb{P}(X^{(0)} = s_0) > 0$,
\[ \mathbb{E}[T | X^{(0)} = s_0] \ge \frac{\ln(s_0) - \ln(s_{\min})}{\delta} \frac{1-\beta}{1+\beta}.\]
\end{theorem}

As a first step towards a proof of Theorem~\ref{thm:oeaOMlower}, we state the following bound on the expected progress, which follows directly from Lemma~6.7 in~\cite{Witt13j}, just taking into account that the \oeares does not sample $0$.
\begin{lemma}
\label{lem:ImprovementBound}
Consider the \oeares with mutation probability $p$ for the maximization of $\onemax$.
Given a current search point with $i$ 0-bits, let $I'$ denote the number of 0-bits in the subsequent search point (after selection). Then we have $\mathbb{E}(i - I') \le \frac{ip}{1 - (1-p)^n} \left( 1 - p + \frac{ip^2}{1-p}\right)^{n-1}$. 
\end{lemma}

Along with the multiplicative drift lower bound theorem (repeated here as Theorem~\ref{thm:multidriftlower}), Theorem~\ref{thm:oeaOMlower} can now be proven as follows.

\begin{proof}
As mentioned above, we follow very closely the proof of Theorem~6.5 in~\cite{Witt13j}. We also use the same notation, just changing the meaning of 0- and 1-bits as Witt considers minimization whereas we regard the maximization. This only has an influence on the notation, not on any of the statements. Using the fact that \onemax is the easiest to optimize pseudo-Boolean linear function (cf.~\cite{DoerrJW12}, the proof easily carries over to the \oeares), it suffices to prove the claimed lower bound for \onemax. 

Let $X^{(t)}$ be the number of 0-bits in the solution at time $t$ and let $\tilde{p} = \max\{p, 1/n \}$. The probability of flipping at least $b = \tilde{p}n\ln n$ bit is superpolynomially small. We can therefore condition on this event not happening in any of the first $n^2$ iterations and pessimistically assume that the runtime of the \oeares is $0$ whenever $b$ or more bits are flipped in the first $n^2$ iterations.

As in~\cite{Witt13j} it is not difficult to argue that with probability $1-o(1)$ the \oeares reaches a state in which the search point in the memory has between $s_{\max}/2$ and $s_{\max}:=1/(2n\tilde{p}^2 \ln n)$ 0-bits. Let $t^*$ be the first point in time that this happens. We bound from below the time needed by the \oeares to reach, starting in the current search point $x^{(t^*)}$ having at most $s_{\max}$ 0-bits, a search point having at most $s_{\min} := n\tilde{p} \ln^2 n$ 0-bits for the first time. 

In order to apply the multiplicative drift theorem, we need to verify the two conditions of Theorem~\ref{thm:multidriftlower}. 
Using that $b = \tilde{p}n \ln n = \beta s_{\min}$ 
and that $X^{(t)}$ only decreases with increasing $t$,
we see that, for all $s_{\min} \leq s \leq s_{\max}$ and for all $t\ge t^{*}$, the probability of making large jumps can be bounded by
\begin{align*}
\mathbb{P}(X^{t} - X^{t+1} \ge \beta s | X^t = s) \le \mathbb{P}(X^{t} - X^{t+1} \ge \beta s_{\min} | X^t = s) = 0 \le \beta \delta / \ln s,
\end{align*}
where in the ``=0'' equality we make use of our condition that the \oeares does not flip more than $b$ bits in the first $n^2$ iterations.

In order to prove the first condition of Theorem~\ref{thm:multidriftlower}, an upper bound on the expected drift, we use that 
$\frac{1-\beta}{1+\beta} = 1 - o(1)$, 
$\ln (s_0/s_{\min}) = (1-o(1))\ln (1/(\tilde{p}^3n^2))$, 
Lemma~\ref{lem:ImprovementBound}, and $1/\tilde{p} \le 1/p$, to see that, for $i \leq s_{\max}$,  
\begin{align*}
\frac{\mathbb{E}(X^{(t)} - X^{(t+1)} | X^{(t)} = i)}{i} &\le \frac{p}{1-(1-p)^n}\left( 1-p + \frac{s_{\max} p^2}{1-p} \right) ^{n-s_{\max}} \\
&\le \frac{p}{1-(1-p)^n}\left( 1-p + \frac{1}{n\ln n} \right) ^{n-s_{\max}} \\
&\le \frac{p}{1-(1-p)^n}\left( (1-p) \left( 1 + \frac{2}{n\ln n} \right) \right) ^{n-s_{\max}} \\
&= (1-o(1))\frac{p(1-p)^n}{1 - (1-p)^n}.
\end{align*}
This shows that the first condition of Theorem~\ref{thm:multidriftlower} holds for $\delta := (1-o(1))\frac{p(1-p)^n}{1 - (1-p)^n}$. 

The multiplicative drift theorem yields a lower bound of 
\[ \mathbb{E}[T_{>0}] \ge (1 - o(1))\frac{1 - (1-p)^n}{p(1-p)^n} \min\{\ln n, \ln (1/(p^3n^2)) \} \]
to go from a solution with at most $s_{\max}$ 0-bits to one with at least $s_{\min}$ 0-bits. 
\end{proof}

The bounds in Theorem~\ref{thm:oeaOM} are monotonically increasing in $p$. Intuitively, the \oeashift and the \oeares converge to RLS when $p$ converges to zero, since the probability distributions for $\ell$ get more and more concentrated around $1$. In fact, 
for $T_p(\OM)$ denoting the upper bound on the expected runtime of either of the two algorithms, 
we observe that 
$\lim_{p \rightarrow 0}{T_p(\OM) \rightarrow nH_n}$, which almost equals the $nH_{n/2} \approx n \ln(n) - 0.1159 n +O(1)$ expected runtime of RLS~\cite{DoerrD16RLS}. Clearly, for $p=0$ the \oeashift equals RLS, so the remaining difference is an artifact of our simple upper bound, not of the algorithm. 

\textbf{Experimental Results.} 
Figure~\ref{fig:oeaOM} reports the average optimization times of 100 independent runs of the \oea, the \oeares, and the \oeashift for $n$ ranging from 100 to $4\,000$ with mutation rate $p=1/n$. Detailed statistical information about these runs are reported in Table~\ref{tab:oeaOM}. We observe that the simple ideas of resampling and shifting probability mass from $0$ to $1$ yields significant performance gains, e.g., $35.6\%$ for the \oeares in comparison with the \oea for $n=4\,000$ and $47.1\%$ for the \oeashift in comparison with the \oea for the same problem size. We also see that, as our upper bounds suggest, the average performance of the \oeashift is better than that of the \oeares. 
\begin{figure}[h!]
\begin{center}
\includegraphics[scale=0.2]{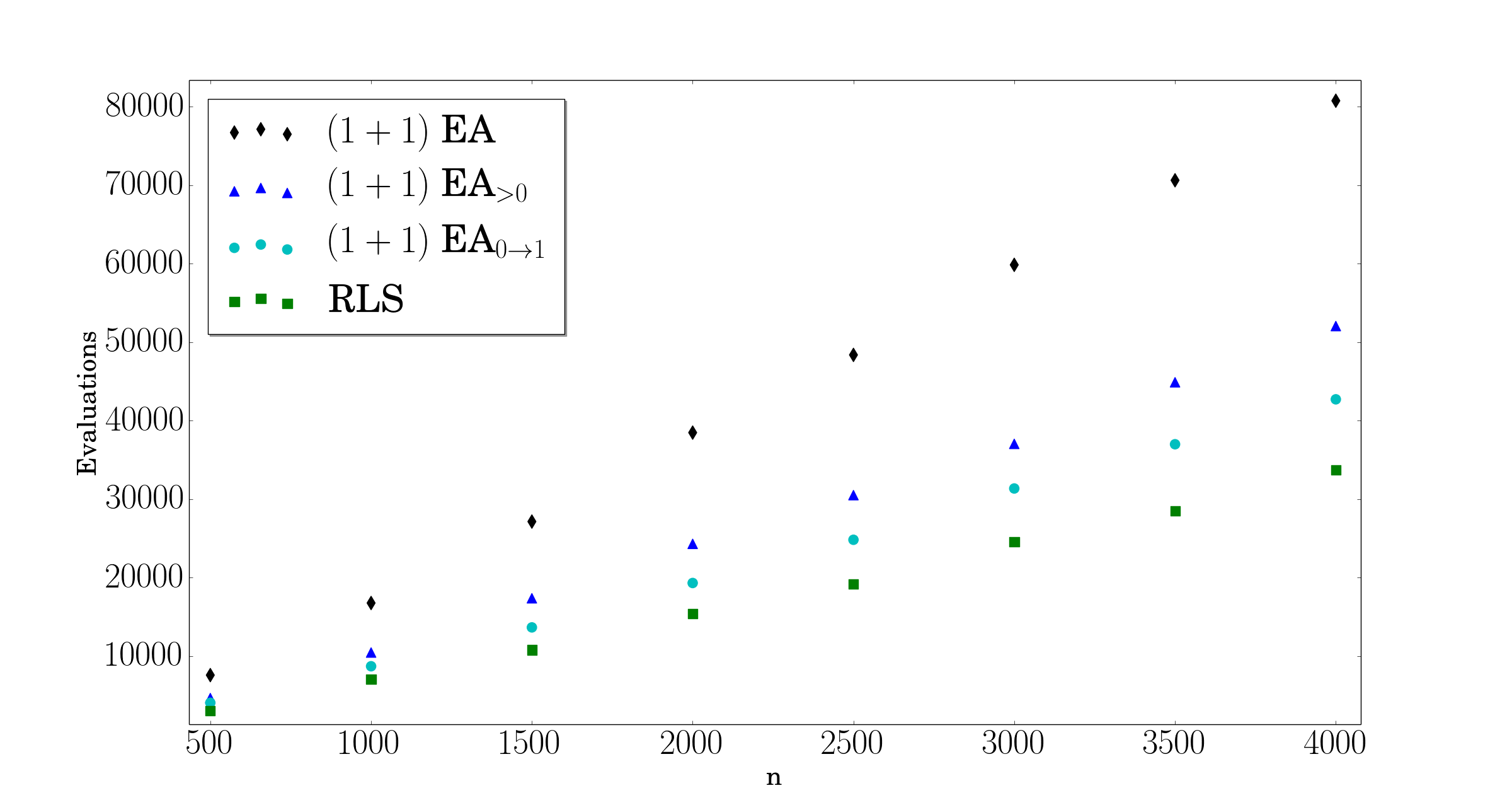}
\end{center}
\caption{Average runtimes for 100 independent runs of the respective algorithms (with mutation probability $p=1/n$ for the \oea and its variants) on \onemax for different problem sizes~$n$. Statistical data is available in Table~\ref{tab:oeaOM}}
\label{fig:oeaOM}
\end{figure}
%


\subsubsection{Performance of the (1+1) EA Variants on Linear Functions}
\label{sec:implementLin}

It is well known since the work of Witt~\cite{Witt13j} that the expected optimization time of the \oea with mutation probability $p=c/n$ on any linear function $f:\{0,1\}^n \rightarrow \R, x \mapsto \sum_{i=1}^n{w_i x_i}$ with $w_1, \ldots, w_n \in \R$ equals 
$(1\pm o(1)) \frac{e^c}{c} n \ln n$, cf. Theorem~3.1(3) in~\cite{Witt13j}. 
With a bit more effort than for the proof of Theorem~\ref{thm:oeaOM} we can easily generalize Witt's result to the \oeares. Using the same approach, similar bounds can also be proven for the \oeashift, but we do not do this explicitly in this work. 

\begin{theorem}\label{thm:eareslin}
Let $c>0$. The expected optimization time of the \oeares with mutation probability $p=c/n$ on any linear function $f:\{0,1\}^n \rightarrow \R, x \mapsto \sum_{i=1}^n{w_i x_i}$ is 
$(1 \pm o(1)) \frac{e^c - 1}{c} n \ln n$.
\end{theorem}

\begin{proof}
The lower bound follows from Theorem~\ref{thm:oeaOMlower} by observing that, just as for the \oea, also for the \oeares the easiest to optimize linear function is \onemax~\cite{DoerrJW12}; i.e., the expected optimization time of the \oeares on an arbitrary linear function is at least as large as that on \onemax. In fact, this lower bound holds more generally for any function with unique global optimum. 

The upper bound can be easily obtained from the proof of Theorem~4.1 in~\cite{Witt13j}. One only has to adjust the probabilities of the events in Witt's proof to account for the fact that the \oeares samples the number of bits to flip from the conditional binomial distribution $\Bin_{>0}(n,p)$ instead of the unconditional $\Bin(n,p)$ one. This changes the expected drift by a multiplicative factor of $1/(1-(1-p)^n)$, thus resulting in a multiplicative $(1-(1-p)^n)$ factor for the runtime estimate. Apart from this small change in the expected drift, the remainder of the proof remains identical. 
\end{proof}


\subsubsection{The (1+1) EA on LeadingOnes.}
\label{sec:implementLO}

We now regard \leadingones, another classical benchmark problem in the theory of EC. \leadingones assigns to each bit string the maximal number of initial ones, i.e., $\LO(x):=\max \{ i \in [0..n] \mid \forall j \leq i: x_j=1\}$.

B\"ottcher, B. Doerr, and Neumann showed in~\cite{BottcherDN10} that the expected optimization time of the \oea with mutation probability $p$ on \leadingones equals $\frac{1}{2p^2}((1-p)^{-n+1}-(1-p))$. This expression is minimized for $p \approx \frac{1.59}{n}$, yielding an expected optimization time of approximately $0.77n^2$.

Intuitively, when we ignore iterations with $\ell=0$, the expected optimization time should just decrease by a multiplicative factor of $1 - (1-p)^n$, just as in the case of \onemax. Building on the proof in~\cite{BottcherDN10}, it is not difficult to show that this intuition is correct. Interestingly, this observation has been previously made in~\cite[Theorem~3]{JansenZ11}. The proof of~\cite{BottcherDN10} can also be used to analyze the performance of the \oeashift.

\begin{theorem}\label{thm:LO}
The expected optimization time $\E[T_{>0,p}(\LO)]$ of the \oeares with mutation probability $p$ for \leadingones equals 
\begin{align*}
\frac{1 - (1-p)^n}{2p^2} ((1-p)^{-n+1}-(1-p)) = \frac{(1 - (1-p)^n)^2}{2p^2(1-p)^{n-1}},
\end{align*} 
while for the \oeashift it holds that
\begin{align*}
\E[T_{\shift,p}(\LO)] = \frac{1}{2}\sum\limits_{j=0}^{n-1}\frac{1}{p(1-p)^{n-j} + \frac{1}{n}(1-p)^n}.
\end{align*} 
\end{theorem}

The full proof of Theorem~\ref{thm:LO} for the \oeares can be found in~\cite[Section~2]{JansenZ11}, and the one for the \oeashift is very similar. We nevertheless sketch the main steps and begin by recalling two central theorems from~\cite{BottcherDN10}. B\"ottcher et al. consider an algorithm to be a \oea-variant if it follows the scheme of Algorithm~\ref{alg:oea}. It is not difficult to see that the following results also apply to the \oeares and the \oeashift. 
\begin{theorem}[Theorem~1 in~\cite{BottcherDN10}]\label{thm:LO1}
Let $x \in \{0,1\}$ be a random point with $\LO(x) < n$. Then for any \oea-variant $\A$, the time $A_{n - \LO(x)}$ needed to wait to find an improvement satisfies
\[ \E[A_{n - \LO(x)}] = \frac{1}{\mathbb{P}[\LO(y) > \LO(x)]}, \]
where $y$ denotes the outcome of one iteration of $\A$.
\end{theorem}

\begin{theorem}[Theorem~2 in~\cite{BottcherDN10}]\label{thm:LO2}
For any any \oea-variant the expected time needed to find the optimum given a random solution with $\LO$-value $n-i$ is
\[ \E[T_i] = \E[A_i] + \frac{1}{2}\sum\limits_{j=0}^{i-1} \E[A_j], \] 
where $A_j$ is the time needed to find an improvement starting in a random search point of $\LO$-value $n-j$. 
\end{theorem}

In this setup, B\"ottcher et al. show that for the \oea with mutation probability $p$ it holds that $\mathbb{P}[\LO(y) > \LO(x)] = (1-p)^{j}p$ if $\LO(x) = j$ and conclude that for a fixed mutation rate, the expected optimization time is given by
\[ \frac{1}{2p^2}((1-p)^{-n+1} - (1-p)). \]
As mentioned above, this expression is minimized for $p \approx 1.59/n$, giving an expected optimization time of $\approx 0.77n^2$. 

As mentioned above, it is quite intuitive that the probability that the \oeares leaves fitness level $j$ equals 
\[ \mathbb{P}[\LO(y) > \LO(x)] = \frac{(1-p)^j p}{1 - (1-p)^n}, \]
because all we have to change is to replace the binomial bit flip probabilities by those of the conditional distribution $\Bin_{>0}(n,p)$. This intuitive argument has been formally proven in~\cite[Theorem~3]{JansenZ11}, resulting in the claimed expected optimization time of the \oeares of
\begin{align}\label{eq:lo}
\mathbb{E}[T] =  \frac{1 - (1-p)^n}{2p^2} ((1-p)^{-n+1}-(1-p)) = \frac{(1 - (1-p)^n)^2}{2p^2(1-p)^{n-1}}.
\end{align}
 
For the \oeashift the probability of improvement at a given iteration starting in a random point $x$ with $\LO(x) = j$ is given by
\[ \mathbb{P}[\LO(y) > \LO(x)] = p(1-p)^j + \frac{1}{n}(1-p)^n, \]
yielding a total expected optimization time for \leadingones of 
\begin{align*}
T_{0 \to 1} 
&= \frac{1}{2}\sum\limits_{j=0}^{n-1}\frac{1}{p(1-p)^{n-j} + \frac{1}{n}(1-p)^n}.
\end{align*}

For $p=1/n$ the \oeares thus needs, on average, about $\frac{(e-1)^2}{2e}n^2 \approx 0.543 n^2$ function evaluations to optimize \leadingones. This value is just slightly above the expected optimization time $n^2/2$ of RLS. We also see that $\lim_{p \rightarrow 0}{\E[T_{>0,p}(\LO)]}=n^2/2$. As for \onemax the reason for this is quite simple: the smaller $p$, the more likely we are to sample $\ell=1$, thus resembling an RLS-iteration. 

The bound for the \oeashift is more difficult to interpret, but as in the case of \onemax the convergence to $n^2/2$ for $p \to 0$ is faster than that of the \oeares, cf. Table~\ref{tab:LOp}, which compares for $n = 1\,000$ the expected optimization times of the \oea, the \oeares, and the \oeashift for different values of $p$. 

\renewcommand{\arraystretch}{1.5}
\begin{table}[h!]
\begin{center}
\begin{tabular}{c|c|c|c}
\hline
$p$ 	& $1/n$ & $1/(10n)$ & $1/(100n)$\\
\hline
$\mathbb{E}[T_p]/n^2$ 
& 0.8589 & 5.2583 & 50.2506\\
$\mathbb{E}[T_{>0,p}]/n^2$ 
& 0.5431 & 0.5004 & 0.5000\\
$\mathbb{E}[T_{0\to 1,p}]/n^2$ 
& 0.5166 & 0.5000 & 0.5000\\
\hline
\end{tabular}
\end{center}
\caption{Comparison of the expected optimization times of the \oea, the \oeares, and the \oeashift with different values of $p$ on \leadingones of problem dimension $n=1\,000$. The poor performance of the \oea stems from the increasing number of iterations in which no bit is flipped at all.}
\label{tab:LOp}
\end{table}

For fixed $p=1/n$ the expected optimization time of the \oeashift converges to $0.5163 n^2$. Table~\ref{tab:LOn} compares the expected optimization times of the three algorithms for different values of $n$. 
\renewcommand{\arraystretch}{1.5}
\begin{table}
\begin{center}
\begin{tabular}{c|c|c|c|c|c|c}
\hline
$n$ 																& 10 & 100 & 1\,000 & 10\,000 & 100\,000 & 1\,000\,000 \\
\hline
$\E[T_{p=1/n}(\LO)]/n^2$ 
& 0.8405 & 0.8573 & 0.8589 & 0.8591 & 0.8591 & 0.8591 \\
$\E[T_{>0,p=1/n}(\LO)]/n^2$ 
& 0.5474 & 0.5435 & 0.5431 & 0.5430 & 0.5430 & 0.5430 \\
$\E[T_{0\to 1,p=1/n}(\LO)]/n^2$ 
& 0.5536 & 0.5197 & 0.5166 & 0.5163 & 0.5163 & 0.5163 \\
\hline
\end{tabular}
\end{center}
\caption{Comparison of the expected optimization times of the \oea, the \oeares, and the \oeashift with $p=1/n$ on \leadingones for different problem dimension $n$}
\label{tab:LOn}
\end{table}

\textbf{Experimental Results.} 
Figure~\ref{fig:oeaLO} shows for 100 independent runs on \leadingones the observed average runtimes of the three different algorithms with mutation rate $p=1/n$ for different problem sizes~$n$. While the \oeares and the \oeashift have a significantly better performance than the \oea already for small problem sizes ($38\%$ to $40\%$ for $n=600$), the difference between the two algorithms is much smaller than for \onemax. 
\begin{figure}[h!]
\begin{center}
\includegraphics[scale=0.2]{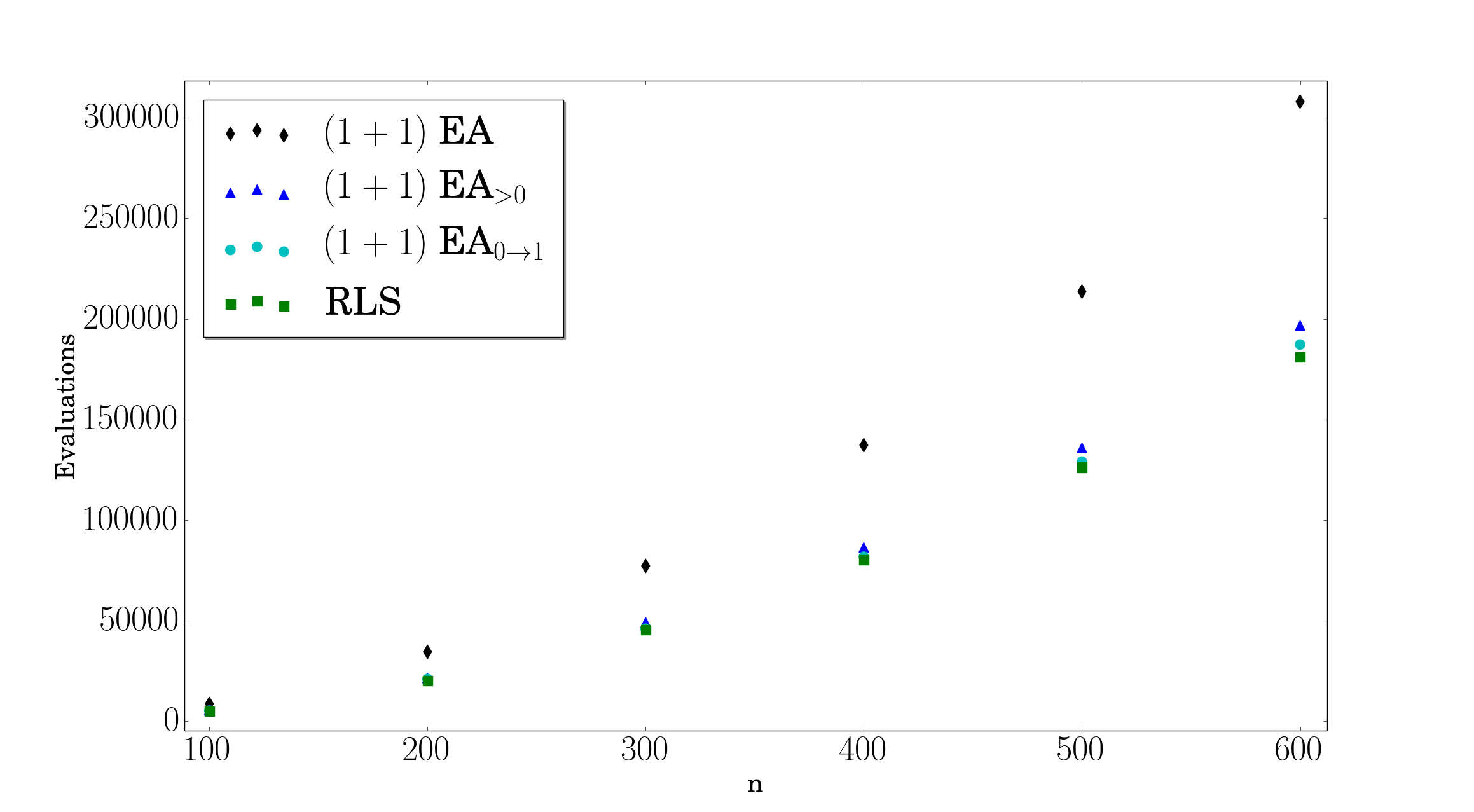}
\end{center}
\caption{Average runtimes for 100 independent runs of the respective algorithms (with mutation probability $p=1/n$ for the \oea and its variants) on \leadingones for different problem sizes~$n$.}
\label{fig:oeaLO}
\end{figure}
%

\subsection{The \texorpdfstring{\ga}{(1+(lambda,lambda))~GA} and the \texorpdfstring{\tpo}{Greedy (2+1)~GA}}
\label{sec:ga-implement}

In the previous section we have made the point that an algorithm should not be charged for iterations in which no bit is flipped. We now discuss that, more generally, one should not count function evaluations in which the sampled offspring equals one of its parents (assuming that this is easy to detect, of course). To this end, we will describe in this section reasonable implementations of the \ga presented in~\cite{DoerrDE15} 
and the \tpo from~\cite{Sudholt12}. 
As mentioned above, we will obtain a quite surprising result, namely that the \tpores, which we obtain from the \tpo by ignoring iterations in which the offspring equals one of its two parents, has an expected optimization time on \onemax that is strictly smaller than that of RLS, and, more than that, smaller than that of \emph{any} other unary unbiased black-box algorithm.  

\subsubsection{The \texorpdfstring{\ga}{(1+(lambda,lambda))~GA}}
\label{sec:ga}
The \ga 
is a theory-inspired EA which introduces a novel use of crossover as a \emph{repair mechanism} to discrete optimization~\cite{DoerrDE15}. It keeps in the memory a best so far solution $x$. Every iteration consists of a mutation and a crossover step. In the \emph{mutation step}, $\lambda$ offspring are created from~$x$. To ensure that all these offspring have the same distance from~$x$, the mutation strength $\ell$ is sampled from the binomial distribution $\Bin(n,p)$ and all offspring are created by the variation operator $\mut_{\ell}$. In the \emph{crossover phase}, the best of these offspring $x'$ (ties broken uniformly at random) is then recombined, in $\lambda$ independent trials, with $x$, using a biased crossover $\cross_c(\cdot,\cdot)$ which, independently for every position $i \in [n]$, chooses the entry from the second argument with probability $c$ and chooses the entry of the first argument otherwise. 
The best of these $\lambda$ recombined offspring (ties again broken uniformly at random), $y$, replaces $x$ if it is at least as good, i.e., if $f(y) \geq f(x)$. 

It was shown in~\cite{DoerrD15tight,Doerr16} that for suitably chosen parameters $\lambda, p, c$ the \ga has an $\Theta(n \sqrt{\log(n) \log\log\log(n) / \log\log(n)})$ expected runtime on \onemax,\footnote{Note that here and for all other results stated in this work we count the number of function evaluations, not the number of generations.} thus beating the $\Omega(n \log n)$ bound valid for any mutation-only algorithm~\cite{LehreW12,DoerrDY16}. The expected performance of the \ga can be further improved to linear by using non-static parameters, cf.~\cite{DoerrDE15,DoerrD15self}. 

From the above description we see that one iteration of the \ga costs $2 \lambda$ function evaluations, as we have to evaluate the $\lambda$ offspring created in the mutation phase and the $\lambda$ offspring created in the crossover phase. When $\ell=0$, all these $2 \lambda$ offspring equal $x$, and thus cause $2 \lambda$ useless function evaluations. Our first suggestion is therefore to sample, as in the \oeares, the mutation strength $\ell$ from the conditional binomial distribution $\Bin_{>0}(n,p)$. 
With the recommended parameter setting\footnote{Intuitively, the recommendation to use $c=1/\lambda$ stems from the observation that a random offspring in the mutation phase has an expected distance of $pn=\lambda$ to $x$. Using $c=1/\lambda$ therefore corresponds to having (roughly, since we have an intermediate selection step) an expected Hamming distance of 1 between $x$ and $y$.} $p=\lambda/n$ and $c=1/\lambda$ the probability to sample $\ell=0$ in the unconditional binomial distribution is $(1-p)^n = (1-\lambda/n)^n \approx \exp(-\lambda)$. Sampling from the conditional distribution thus saves us an expected number of $2\exp(-\lambda)\lambda$ function evaluations per iteration. 
Our second suggestion concerns the crossover phase. Depending on $\lambda$, whose optimal value approaches $\sqrt{n}$ as fitness increases, the probability $c=1/\lambda$ to take an entry from $x'$ can be quite small. It is therefore not unlikely that an offspring created in the crossover phase equals one of its two parents, in particular the original parent $x$. Since this equality can be easily checked, we suggest not to evaluate such offspring. Finally, when the winner $x'$ of the mutation phase is better than that of the crossover phase (i.e., if $f(x')>f(y)$), we suggest to replace $x$ by $x'$ if $f(x') \geq f(x)$. 


That our suggested changes are indeed practice-driven can be seen by looking at the implementation of the \ga reported in~\cite{GoldmanP15}, which is available on GitHub~\cite{GoldmanGITHUB}. Indeed, all of the suggested changes have been implemented there. Our \gaopt ignores, however, some additional problem-driven changes made in~\cite{GoldmanP15}. 

For the \ga on \onemax only asymptotic runtime bounds are available~\cite{DoerrD15tight,Doerr16,DoerrDE15}. 
We can therefore at the moment not compute the optimal parameter values of the \gaopt. For our experiments we use the self-adjusting choice of $\lambda$ proposed and analyzed in~\cite{DoerrD15self}, both for the \ga as well as for the \gaopt. 
This self-adjusting choice yields linear expected runtime on \onemax and works as follows. In the beginning, $\lambda$ is initialized as $\lambda=1$. At the end of each iteration, it is checked if the iteration was successful. If so, i.e., if $f(y)>f(x)$, then $\lambda$ is decreased to $\lambda/F$, and it is increased to $\lambda\cdot F^{1/4}$ otherwise. For our experiments we use $F=1.5$. With this self-adjusting rule the \gaopt becomes Algorithm~\ref{alg:GAoptself}.  

\begin{algorithm2e}[t]%
	\textbf{Initialization:} 
	Choose $x \in \{0,1\}^n$ u.a.r. and evaluate $f(x)$\;
	Initialize $\lambda \assign 1$\;
 \textbf{Optimization:}
\For{$t=1,2,3,\ldots$}{
\underline{\textbf{Mutation phase:}}\\
\Indp
Sample $\ell$ from $\Bin_{>0}(n,p)$\;
\lFor{$i\in [\lambda]$}{$x^{(i)} \assign \mut_{\ell}(x)$; Evaluate $f(x^{(i)})$}
Choose $x' \in \{x^{(i)}\mid i \in [\lambda]\}$ with $f(x')=\max\{f(x^{(i)})\mid i \in [\lambda]\}$ u.a.r.\;
\Indm
\underline{\textbf{Crossover phase:}}\\
\Indp
\lFor{$i\in [\lambda]$}{$y^{(i)} \assign \cross_{c}(x,x')$; \textbf{if} {$y^{(i)} \notin \{ x,x'\}$} \textbf{then} {evaluate $f(y^{(i)})$}}
Choose $y \in \{y^{(i)}\mid i \in [\lambda]\} \cup \{x'\}$ with $f(y)=\max\left(\{f(y^{(i)})\mid i \in [\lambda]\} \cup \{f(x')\} \right)$ u.a.r.\;
\Indm
\underline{\textbf{Selection and update step:}}\\
\Indp
\lIf{$f(y)>f(x)$}{
$x \assign y$; $\lambda \assign \max\{\lambda/F,1\}$}
\lIf{$f(y)=f(x)$}{
$x \assign y$; $\lambda \assign \min\{\lambda F^{1/4},n\}$}
\lIf{$f(y)<f(x)$}{$\lambda \assign \min\{\lambda F^{1/4},n\}$}
\Indm 
}
\caption{The self-adjusting \gaopt, maximizing $f:\{0,1\}^n \to \R$, with offspring population size~$\lambda$, mutation probability~$p$, crossover bias~$c$, and update strength~$F$. 
}
\label{alg:GAoptself}
\end{algorithm2e}

\subsubsection{The \texorpdfstring{\tpo}{Greedy (2+1)~GA}}
\label{sec:tpo}

As in~\cite{DoerrDE15}, we compare the \ga and the \gaopt with the \tpo from Sudholt~\cite{Sudholt12}. 
The \tpo, or more generally, the greedy $(\mu+1)$~GA presented in~\cite{Sudholt12} maintains a population $\Pe$ of $\mu$ individuals. $\Pe$ is initialized by sampling $\mu$ search points independently and uniformly at random. Each iteration consists of two steps, a crossover step and a mutation step. In the crossover step two parents $x,y$ are selected uniformly at random (with replacement) from those individuals $u \in \Pe$ for which $f(u) = \max_{v \in \Pe}{f(v)}$ holds. Note that if there is only one such search point, then this one is selected twice. From these two search points an offspring $z'$ is created by uniform crossover $\cross_{1/2}(x,y)$. This offspring $z'$ is then mutated by standard bit mutation, i.e., each bit is flipped independently with some probability $p \in [0,1]$. The so-created offspring $z$ is evaluated. If $z \notin \Pe$ and its fitness is at least as good as $\min_{v \in \Pe} f(v)$, it replaces the worst individual in the population, ties broken uniformly at random. The requirement $z \notin \Pe$ is a so-called \emph{diversity mechanism}. 

It is not difficult to see that from the whole population only those with a best-so-far fitness value are relevant, the others are never selected for reproduction. Furthermore we see that in the case that $\mu=2$, even if there are two different individuals $x \neq y$ in the population $\Pe$, the probability to select both of them for reproduction is only $1/2$. In all other cases the crossover phase just reproduces one of the two parents. We change this in our implementation and enforce that in the crossover phase both parents are selected if they have an equal fitness value. As we shall discuss below, this change does not influence our upper bounds much (it changes the constant in the linear term of the overall $\Theta(n \ln n)$ expected runtime, but does not affect the leading constant of the $n \ln(n)$ term), but we believe that in particular for $\mu=2$ this variant is more ``natural''. 
Sudholt showed for his \tpo that its expected optimization time on \onemax is at most 
\begin{align*}
\frac{\ln(n^2p+n)+1+p}{p(1-p)^{n-1}(1+np)}+\frac{8n}{(1-p)^n}.
\end{align*} 
It is very easy to modify his proof to show that the expected optimization time of the \tpo with the new parental selection is at most
\begin{align*}
\frac{\ln(n^2p+n)+1+p}{p(1-p)^{n-1}(1+np)}+\frac{4n}{(1-p)^n}, 
\end{align*} 
i.e., an additive term of $\frac{4n}{(1-p)^n}$ smaller than the original \tpo.

We now modify the \tpo in a similar way as we did for the \ga, cf. Algorithm~\ref{alg:tpores}. 
Our first modification is that we do not evaluate $z$ if it equals one of its parents $x$ or $y$. 
Our second modification concerns the mutation phase. When $z' \in \{x,y\}$, we enforce a mutation strength greater than $0$ by sampling from the conditional distribution $\Bin_{>0}(n,p)$. Note that $z' \in \{x,y\}$ holds when the crossover did not happen (i.e., if $f(x)>f(y)$ in line~\ref{line:xoverselect} of Algorithm~\ref{alg:tpores}) or when the random choices of the crossover resulted in a string that equals one of its two parents. If we denote by $d$ the Hamming distance of $x$ and $y$ this latter event happens with probability $1/2^{d-1}$. The situation $d=H(x,y)=2$ occurs quite frequently, resulting in a $1/2$ probability that the crossover reproduces one of the two parents. 
 
\begin{algorithm2e}[t]%
	Choose $x$ and $y$ from $\{0,1\}^n$ independently and u.a.r. and evaluate $f(x),f(y)$\;
\For{$t=1,2,3,\ldots$}{
		Ensure $f(x) \geq f(y)$ by renaming them if needed\;
		\lIf{$f(x) = f(y)$\label{line:xoverselect}} 
			{$z' \assign \cross_{1/2}(x,y)$;
			\textbf{else} {$z' \assign x$}}
		\lIf{$z' \notin \{x,y\}$} 
				{Sample $\ell$ from $\Bin(n,p)$; 
				\textbf{else} {Sample $\ell$ from $\Bin_{>0}(n,p)$}}
		$z \assign \mut_{\ell}(z')$\; 
		\If{\label{line:diversity}$z \notin \{x,y\}$}{
			evaluate $f(z)$\;
				\lIf{$f(z)\geq f(y)$ and $f(y)<f(x)$}{$y \assign z$}
				\lIf{$f(z)\geq f(y)=f(x)$}{replace either $y$ or $x$ by $z$, chosen u.a.r.}
				}
}
\caption{The \tpores with mutation probability $p$ maximizing a given function $f : \{0,1\}^n \to \R$}
\label{alg:tpores}
\end{algorithm2e}

Following very closely the proof of Theorem~2 in~\cite{Sudholt13}, it is not difficult to obtain the following runtime statement.
\begin{theorem}
\label{thm:tpo}
The expected optimization time of the \tpores with mutation rate $p$ is at most 
$
\frac{(1-(1-p)^n)(\ln(n^2p+n)+1+p)}{p(1-p)^{n-1}(1+np)}+\frac{4n}{(1-p)^n}.
$
For $p=c/n$ the expected optimization time is thus at most
\begin{align}
\label{eq:tpomutrate}
\frac{(1-(1-c/n)^n)}{c(1-c/n)^{n-1}(1+c)} n \ln(n)+\Theta(n).
\end{align}
\end{theorem}
\begin{proof}
Following~\cite{Sudholt12}, we say that the algorithm is on fitness level $i$ if the best individual in the population has fitness value $i$. Like Sudholt, we distinguish two cases. 

\textbf{Case $i$.1: $i=f(x)$ and either $f(x)>f(y)$ or $x=y$.} In this situation there is no crossover. The offspring $z$ is the outcome of standard bit mutation on $x$. The algorithm leaves this situation when (a) $f(z)>i$ or (b) $f(z)=f(x)$ and $z \neq x$. 
The probability for (a) to happen is at least $(n-i)p(1-p)^{n-1}/(1-(1-p)^n)$, since this is the probability that exactly one of the zero bits is flipped in the mutation phase. Likewise, the probability of event (b) is $i (n-i) p^2 (1-p)^{n-2}/(1-(1-p)^n)$.
Once the algorithm has left case i.1 it does never return to it. This is ensured by the diversity mechanism, which allows to include $z$ in the population only if it isn't yet (line~\ref{line:diversity} of Algorithm~\ref{alg:tpores}). The total expected time spent in the cases $i$.1, $i=0,\ldots, n-1$ is therefore at most 
\begin{align*}
	\frac{1-(1-p)^n}{p(1-p)^{n-1}} \sum_{i=0}^{n-1}{\frac{1}{(n-i)(1+ip)}}.
\end{align*}
The same algebraic computations as in~\cite{Sudholt12} show that this expression can be bounded from above by
\begin{align*}
	\frac{(1-(1-p)^n) (\ln(pn^2+n)+1+p)}{p(1-p)^{n-1}(1+np)}.
\end{align*}

\textbf{Case $i$.2: $i=f(x)=f(y)$ and $x \neq y$.} In this case the Hamming distance of $x$ and $y$ is even. Let $X$ denote the number of ones in the intermediate offspring $z'$ in those $2d$ positions in which $x$ and $y$ differ. $X$ is binomially distributed with $2d$ trials and success probability $1/2$. It is not difficult to show that $\Pr[X>d] \geq 1/4$. When $X>d$, then $z' \notin \{x,y\}$ and the mutation strength is therefore sampled from the binomial distribution $\Bin(n,p)$. The probability to sample a zero is $(1-p)^n$. Thus, the probability to leave fitness level $i$ is at least $(1-p)^n/4$ and the total expected time spent in the cases $i$.2, $i=0,\ldots, n-1$ is at most $4n/(1-p)^n$.
\end{proof}
%
For large $n$, we can approximate expression~\eqref{eq:tpomutrate} by $\frac{(1-\exp(-c))}{c\exp(-c)(1+c)} n \ln(n)+\Theta(n)$, which is minimized for $c \approx 0.773581$, yielding an expected optimization time of approximately $(1+o(1)) 0.850953 n \ln(n)$ for the \tpores, cf. Table~\ref{tab:ctpo} for the values of $c$ minimizing $\frac{1-(1-c/n)^n}{c(1-c/n)^{n-1}(1+c)}$ for different values of $n$. For comparison, the expected optimization time of RLS is $(1\pm o(1)) n \ln(n)$, and so is the expected optimization time of the best possible unary unbiased black-box algorithm~\cite{DoerrDY16}.\footnote{In intuitive terms, the class of unary unbiased black-box algorithms, introduced in~\cite{LehreW12}, contains all mutation-based black-box algorithm.} 
This is a quite remarkable result, as it seems to be the first time that a ``classic'' GA is shown to outperform RLS on \onemax.

\renewcommand{\arraystretch}{1.5}
\begin{table}[t]
\begin{center}
\begin{tabular}{c|c|c|c|c|c}
\hline
$n$ 																			& 10			 & 100 & 500 & $1\,000$ & $5\,000$\\
\hline
$c$ 														 					& 0.783953  & 0.774577   & 0.773778   & 0.773679 & 0.773599\\
$\frac{1-(1-c/n)^n}{c(1-c/n)^{n-1}(1+c)}$	& 0.831839  & 0.859091   & 0.850581   & 0.850766 & 0.850915\\
\hline
\end{tabular}
\end{center}
\caption{Values of $c$ minimizing expression~\eqref{eq:tpomutrate} for different values of $n$}
\label{tab:ctpo}
\end{table}

It is beyond the scope of this work to analyze the tightness of the upper bounds proven in Theorem~\ref{thm:tpo}, and additional gains may be possible by choosing different values for $p$. Our empirical results suggest that the upper bound of Theorem~\ref{thm:tpo} is indeed rather weak. We also remark that \RLSopt (described below), the RLS-variant from~\cite{DoerrDY16} achieving the (up to lower order term) optimal run time among all unary unbiased black-box algorithms on \onemax, uses fitness-dependent mutation rates. It is possible (and likely) that the \tpores, as well, could profit further from choosing the mutation rate in such an adaptive way. We have to leave this for future work. 

\RLSopt is essentially RLS, with the only difference that in the mutation step, instead of flipping always one random bit, more than one bit can be flipped. Intuitively, the optimal number $\ell^*_{v}$ of bits to flip depends only on the fitness value $v=\OM(x)$ of the current-best individual $x$ and should be the one that maximizes the expected progress $\E[ \max\{\OM(y)-\OM(x), 0\} \mid \OM(x)=v, y = \mut_{\ell}(x)]$. That this \emph{drift maximizer} is indeed (at least up to the mentioned additive $O(n)$ term) optimal has been formally proven in~\cite{DoerrDY16}. More precisely, it is shown that the expected runtime of \RLSopt on
\OneMax and the unary unbiased black-box complexity of
\OneMax both are $n \ln(n)- cn \pm o(n)$ for a constant $c$ between
0.2539 and 0.2665. 
To run \RLSopt in our experiments, we have computed, for every tested $n$ and every fitness value $v \in [0..n-1]$ the value $\ell^*_{v}$ that maximizes the expected drift 
\begin{align}
\label{eq:RLSdrift}
	\nonumber B(n,v,\ell) & :=
	\E[ \max\{\OM(y)-\OM(x), 0\} \mid \OM(x)=v, y = \mut_{\ell}(x)] \\
	& = 
	\sum_{i=\lceil \ell/2 \rceil}^{\ell}
	\frac{\binom{n-v}{i}\binom{v}{\ell-i}\left(2i-\ell\right)}{\binom{n}{\ell}},
\end{align}
i.e., we do not work with the approximation proposed in~\cite{DoerrDY16} but the original drift maximizer.

\subsubsection{Experimental Results}
\begin{figure}[t]
\begin{center}
\includegraphics[scale=0.2]{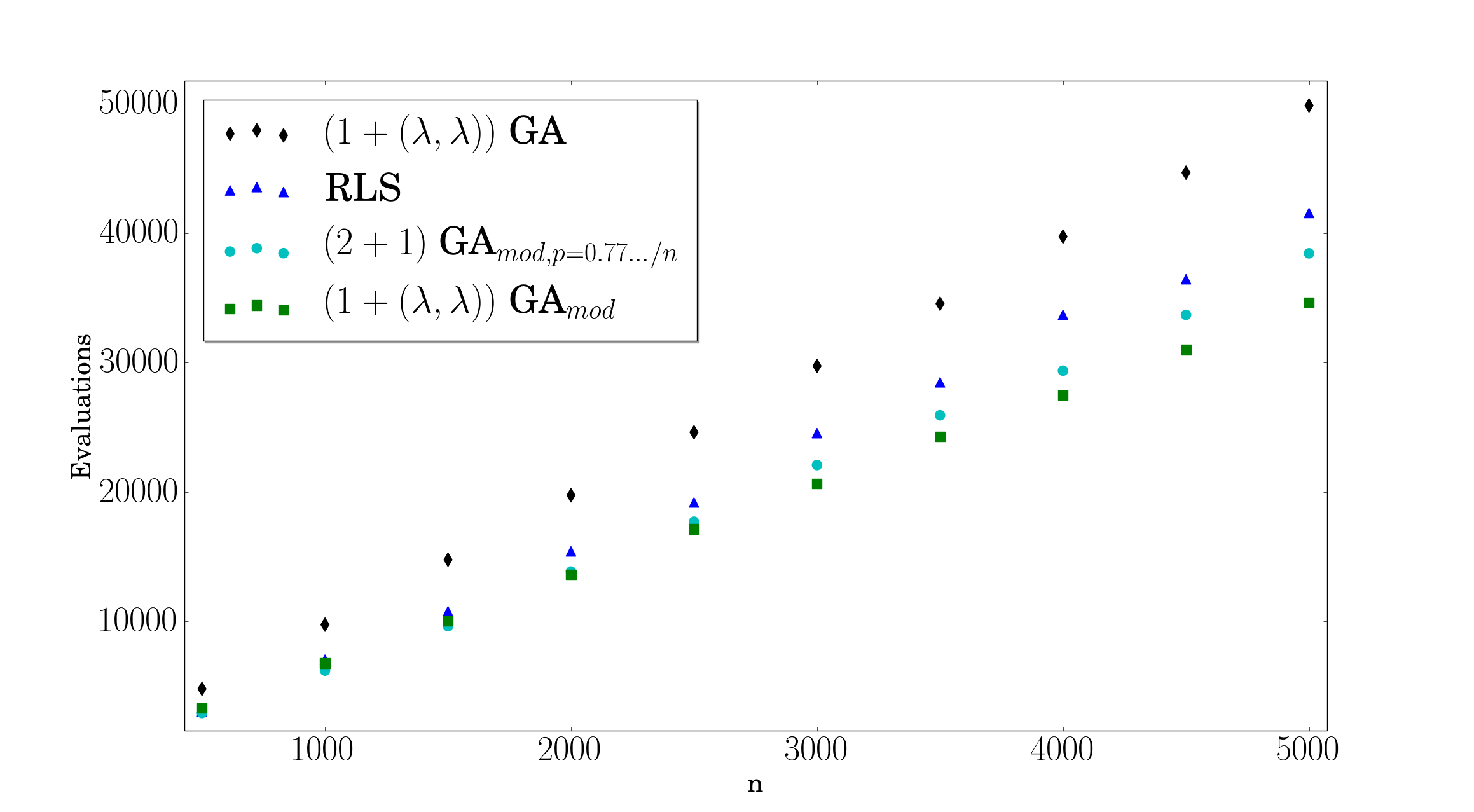}
\end{center}
\caption{Average runtimes for 100 independent runs of the respective algorithms on \onemax for different problem sizes~$n$.}
\label{fig:ga}
\end{figure}
Figure~\ref{fig:ga} shows experimental data for the performance of the mentioned algorithms on \onemax, for $n$ ranging from $100$ to $5\,000$. The \ga and the \gaopt use the self-adjusting $\lambda$ values, while for 
the \tpores we use mutation rate $0.773581/n$. In the reported ranges, the expected performance of the \tpo with mutation rate $p=(1+\sqrt{5})/(2n)$ is very similar to that of the self-adjusting \ga (cf. Figure~8 in~\cite{DoerrDE15}); we do not plot these data points to avoid an overloaded plot. Detailed statistical information for these data points is given in Table~\ref{tab:ga}. We observe that both the \gaopt as well as the \tpores are better than RLS already for quite small problem sizes. We also observe that, in line with the theoretical bounds, the advantage of the \gaopt over the \tpores and over RLS increases with the problem size. 

\section{Runtime Profiles}
\label{sec:profiles}

Most runtime results in discrete EC are statements about \emph{first hitting times}, understood as the time needed by an algorithm until it evaluates for the first time an optimal solution of the underlying problem. In particular the expected value of this random variable is studied. However, in almost all practical applications, the user does not know when the algorithm is ``done''. And even if this could be detected, it may take too long for this event to happen. It is therefore highly relevant to understand how the algorithms perform over time. Jansen and Zarges~\cite{JansenZ14}, for this reason, suggested to adopt a \emph{fixed-budget perspective}, analyzing the expected fitness value that an algorithm has achieved after a fixed number of iterations. Here in this work we suggest a complementary view. 

Instead of reporting only the expected time needed to hit, for the first time, an optimal solution, we suggest to include in the runtime statements the expected time needed to hit intermediate fitness values. When canonical fitness levels exits, such as in the case of \onemax, \leadingones, royal road, and several other functions, we suggest to use these. For other functions, such as linear functions or weighted combinatorial graph problems, the analysis of the expected optimization time often identifies useful to report target fitness levels. In the absence of these, a linear interpolation of the minimal and maximal fitness value could be used. We call such statements \emph{runtime profiles.} We emphasize the fact that runtime profiles have, very naturally, been reported in many empirical works on heuristic optimization. We are thus not suggesting any new concept here. Our intention is rather to highlight to researchers working on theoretical aspects in evolutionary computation that, beyond being possibly more relevant for practitioners, explicitly reporting runtime profiles can give quite interesting insights into the performance of EAs. 

\begin{figure}[t]
\begin{center}
\includegraphics[scale=0.2]{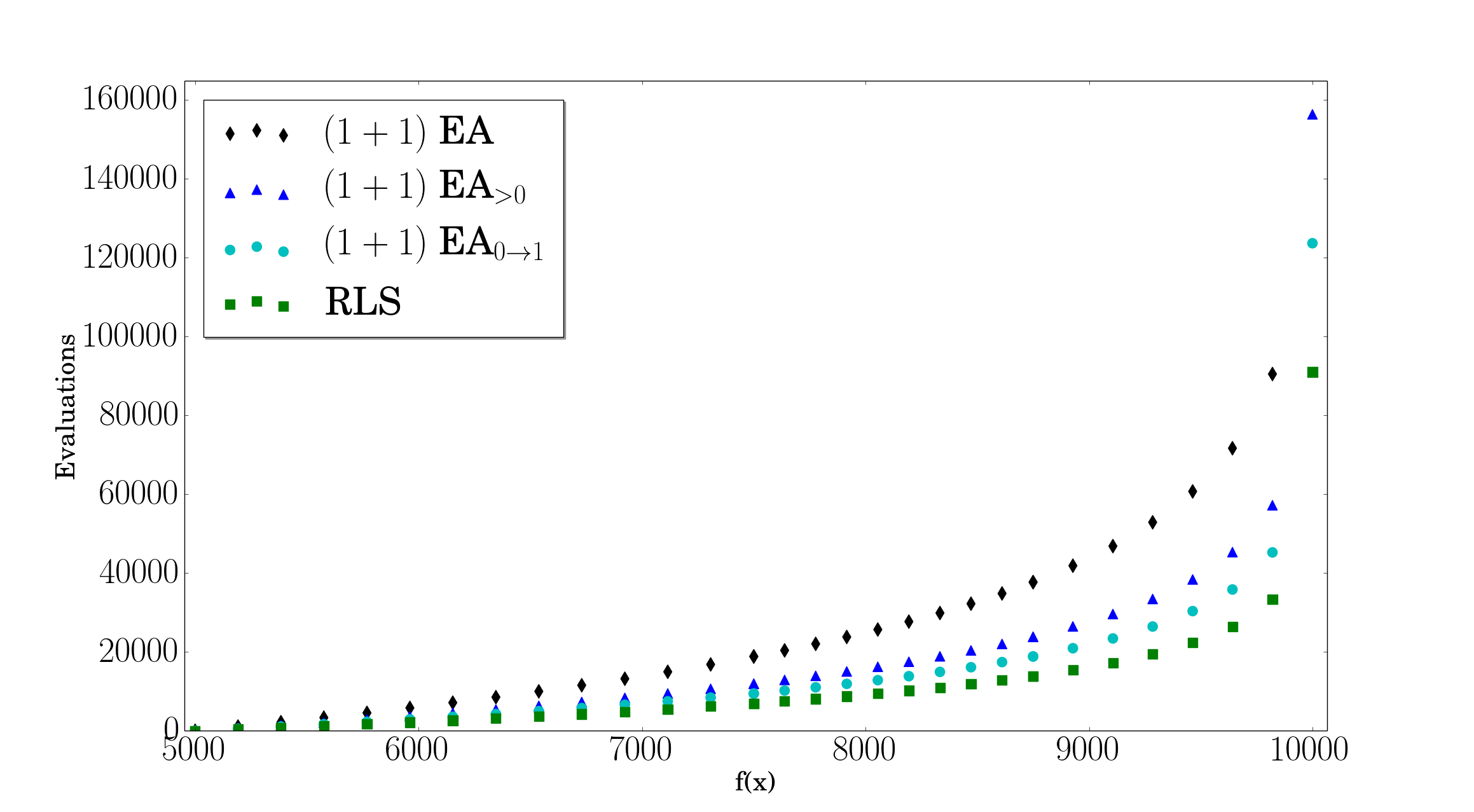}
\end{center}
\caption{Runtime profiles on \onemax for problem size~$n=10\,000$ according to the bounds computed in Theorem~\ref{thm:profileOM} but assuming to start in a search point of fitness value $n/2$.}
\label{fig:theoOMprofile}
\end{figure}

Note that, in contrast to the previous section, our suggestion does not \emph{change} any of the algorithms nor the runtime bounds. We merely suggest to report them in a different way. All bounds reported below can be easily obtained from previous works and are more or less explicit in previous proofs.
%

\begin{theorem}
\label{thm:profileOM}
Let $p \in (0,1)$ and $k \in [n]$. Starting in the all-zeros string $x=(0,\ldots,0)$ the expected time needed to reach for the first time a search point $x$ of \onemax-value at least $k$ is at most 
\begin{itemize}
	\item $n(H_n-H_{n-k})$ for RLS,
	\item $\frac{1}{p(1-p)^{n-1}} (H_n - H_{n-k})$ for the \oea with mutation probability $p$, 
	\item $\frac{1-(1-p)^n}{p(1-p)^{n-1}}(H_n - H_{n-k})$ for the \oeares with mutation probability $p$, and
	\item $\frac{1}{(1-p)^{n-1}(p+1/n-p/n)}(H_n - H_{n-k})$ for the \oeashift with mutation probability $p$. 
\end{itemize}
\end{theorem}
For $p=1/n$ these bounds are $(1\pm o(1)) C n (H_n - H_{n-k})$, where 
$C=e$ for the \oea, 
$C=e-1$ for the \oeares, and 
$C=e/(2-\frac{1}{n})$ for the \oeashift. 

Note that the theorem bounds the time needed to reach fitness level $k$ when starting in a search point of fitness 0. This is a very pessimistic view. In a typical run of these algorithms, already the first search point has an expected fitness of roughly $n/2$ and the probability that it is less than $\frac{n}{2}-\sqrt{n}$ is $o(1)$. In consequence, the empirically observed time needed to reach fitness value $k>n/2$ equals roughly the above-stated bounds minus the time needed to reach fitness level $n/2$, cf. also~\cite{DoerrD16RLS} where it is formally proven that for RLS the total expected optimization time is almost identical to that deterministically starting in a search point of fitness $n/2$. In Figure~\ref{fig:theoOMprofile} we plots the computed runtime profiles for $n=10\,000$, where we assume to start in a search point of fitness $n/2$, i.e., we subtract from the bounds in Theorem~\ref{thm:profileOM} the time needed to reach fitness level $n/2$. We see that the performance of the \oeares, the \oeashift and RLS are quite close for most of the intermediate levels, much closer than what the total optimization time might suggest. 

\begin{figure}[t]
\begin{center}
\includegraphics[scale=0.2]{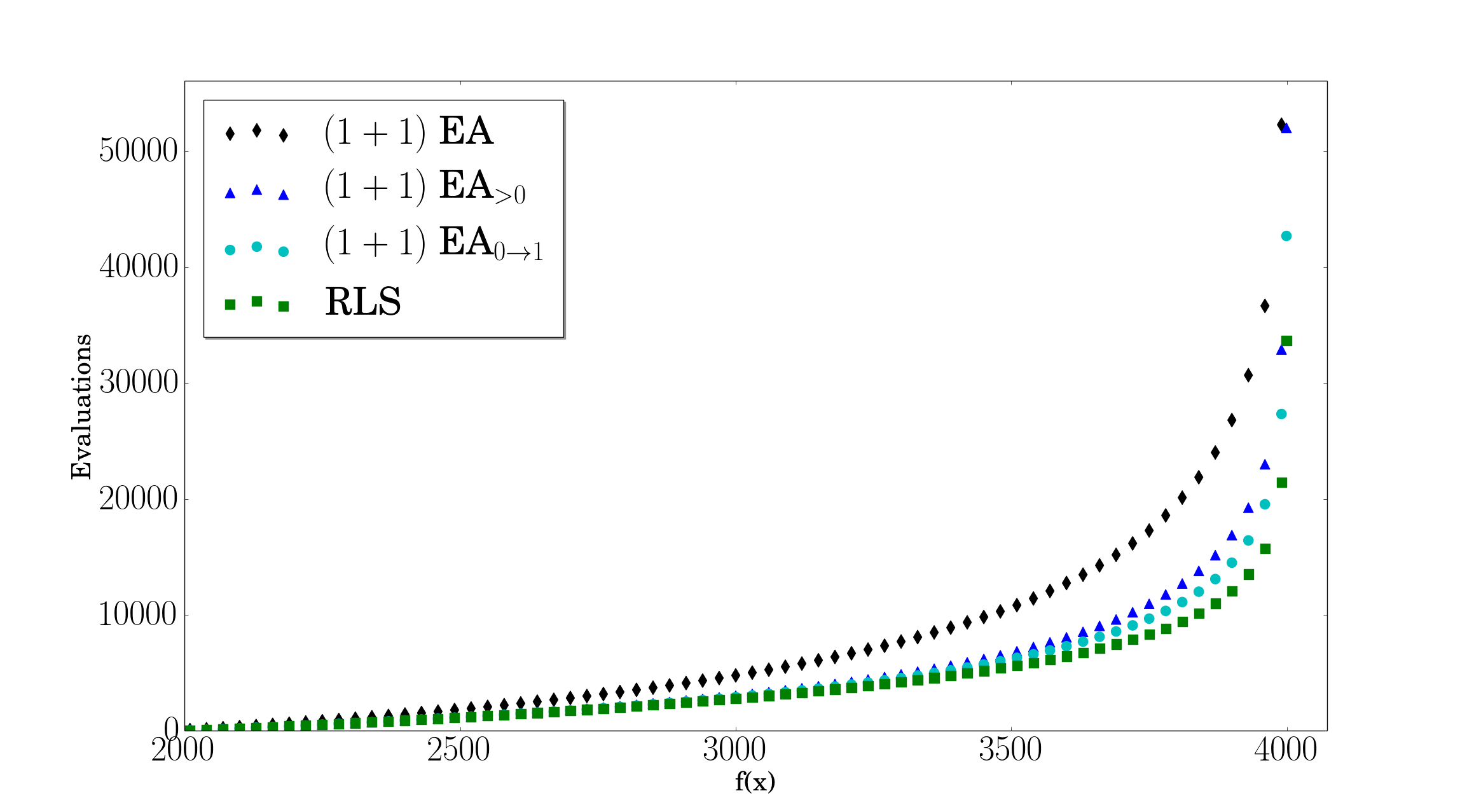}
\end{center}
\caption{Runtime profile for 100 independent runs of the respective algorithms on \onemax for problem size~$n=4\,000$.}
\label{fig:OMprofile}
\end{figure}
In Figure~\ref{fig:OMprofile} we plot the empirical runtime profiles for 100 independent runs of the different algorithms (with mutation rate $p=1/n$ for the \oea-variants) on \onemax with $n=4\,000$. The behavior is much similar to our predicted one from Figure~\ref{fig:theoOMprofile}. In particular, we observe that for all target values $i$, the \oea needs longest, on average, to reach this fitness level. We also easily see from this plot that all four algorithms easily make progress in the beginning. The waiting time for fitness improvement increases with increasing fitness values. While the expected optimization times of RLS, the \oeares with mutation rate $p=1/n$, and the \oeashift with the same mutation rate differs significantly for $n=4\,000$, the expected time to reach the intermediate fitness values is not that diverse for all but the last $10\%$ of the target fitness values.

While for \onemax RLS seems to be consistently better for all relevant intermediate fitness levels, the situation for \leadingones is quite different. 
\begin{theorem}
\label{thm:profileLO}
Let $p \in (0,1)$ and $k \in [n]$. The expected time needed to reach for the first time a search point $x$ of \leadingones-value at least $k$ is at most 
\begin{itemize}
	\item $kn/2$ for RLS,
	\item $\frac{1}{2p^2}((1-p)^{1-k} - (1-p))$ for the \oea with mutation probability $p$, 
	\item $\frac{1-(1-p)^n}{2p^2}((1-p)^{1-k} - (1-p))$ for the \oeares with mutation probability $p$, 
	\item $\frac{1}{2} \sum\limits_{i=n-k+1}^{n} \frac{1}{p(1-p)^{n-i} + \frac{1}{n}(1-p)^n}$ for the \oeashift with mutation probability $p$.
\end{itemize}
\end{theorem}

\begin{figure}
\begin{center}
\includegraphics[scale=0.2]{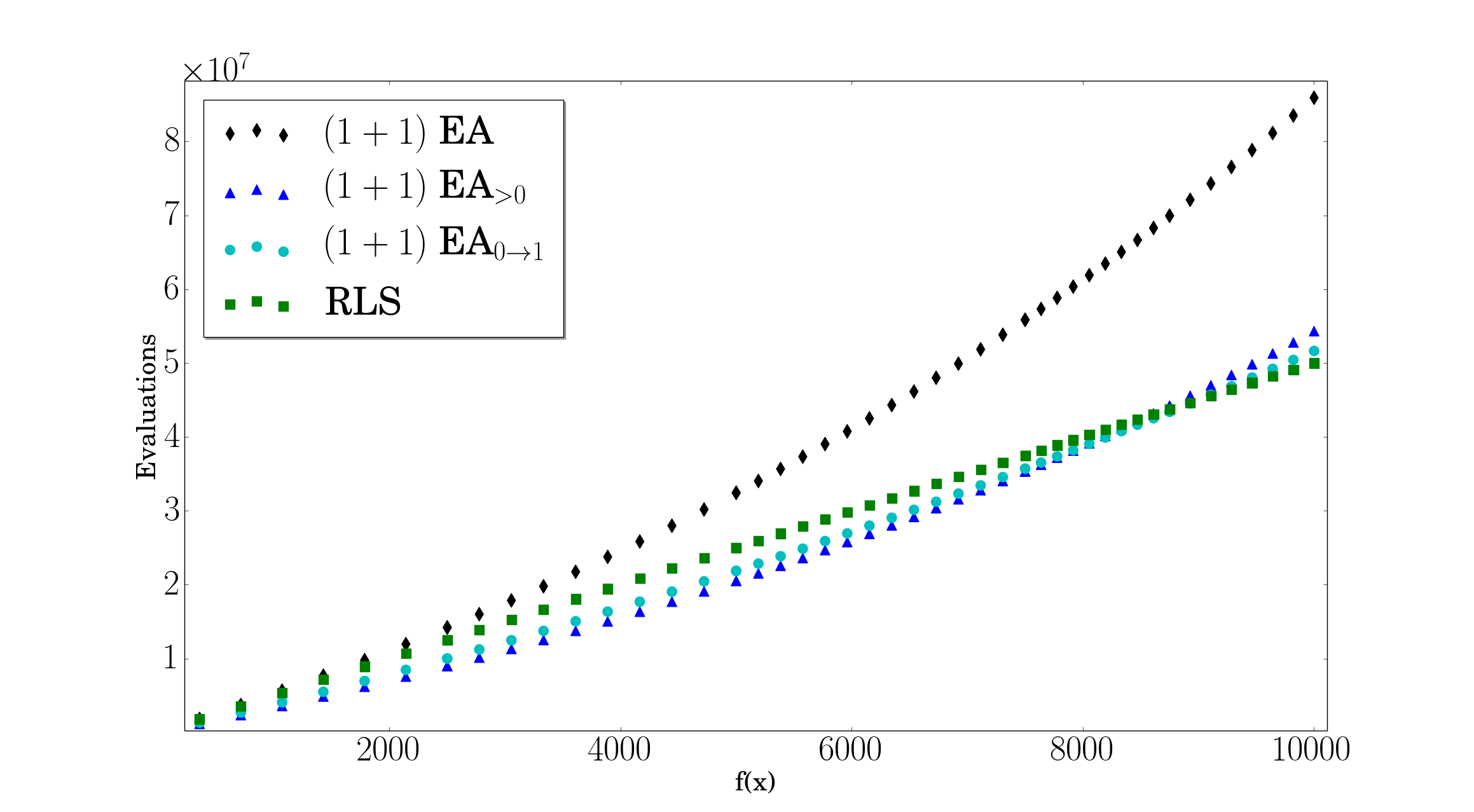}
\end{center}
\caption{Runtime profile of the respective algorithms (with mutation rate $p=1/n$ for the \oea and its variants) on \leadingones for problem size~$n=10\,000$ according to Theorem~\ref{thm:profileLO}.}
\label{fig:theo-LO-profile}
\end{figure}
As above we plot these computed runtime profiles in Figure~\ref{fig:theo-LO-profile}. We observe that the \oeares is the best of all four algorithms for intermediate fitness levels $\leq 7\,980$. The \oeashift has the smallest expected runtime to reach the intermediate fitness values between $7\,980$ to $8\,998$, while RLS is the best of the four algorithms only for fitness levels $i>8\,999$. RLS is faster than the \oeares for intermediate fitness values $i>8\,566$. As mentioned above, such insights are very important for the design of parameter/operator selection schemes. 

In Figure~\ref{fig:LO-profile} we show the  empirically observed runtime profiles for 100 independent runs of the different algorithms on \leadingones with $n=500$. 
We observe that, as our theoretical bounds suggest, the \oeares (\oeashift) needs less iterations in expectation than RLS to reach fitness values $i \le 430$ ($i\le 460$), while RLS, on average, reaches fitness levels $i>430$ ($i>460$) faster. The theoretical bounds in Theorem~\ref{thm:profileLO} suggest cut-off points at $i=429$ ($i=449$), matching our empirical findings quite well. Statistical information for this experiment can be found Table~\ref{tab:LOprofile}. 
\begin{figure}
\begin{center}
\includegraphics[scale=0.2]{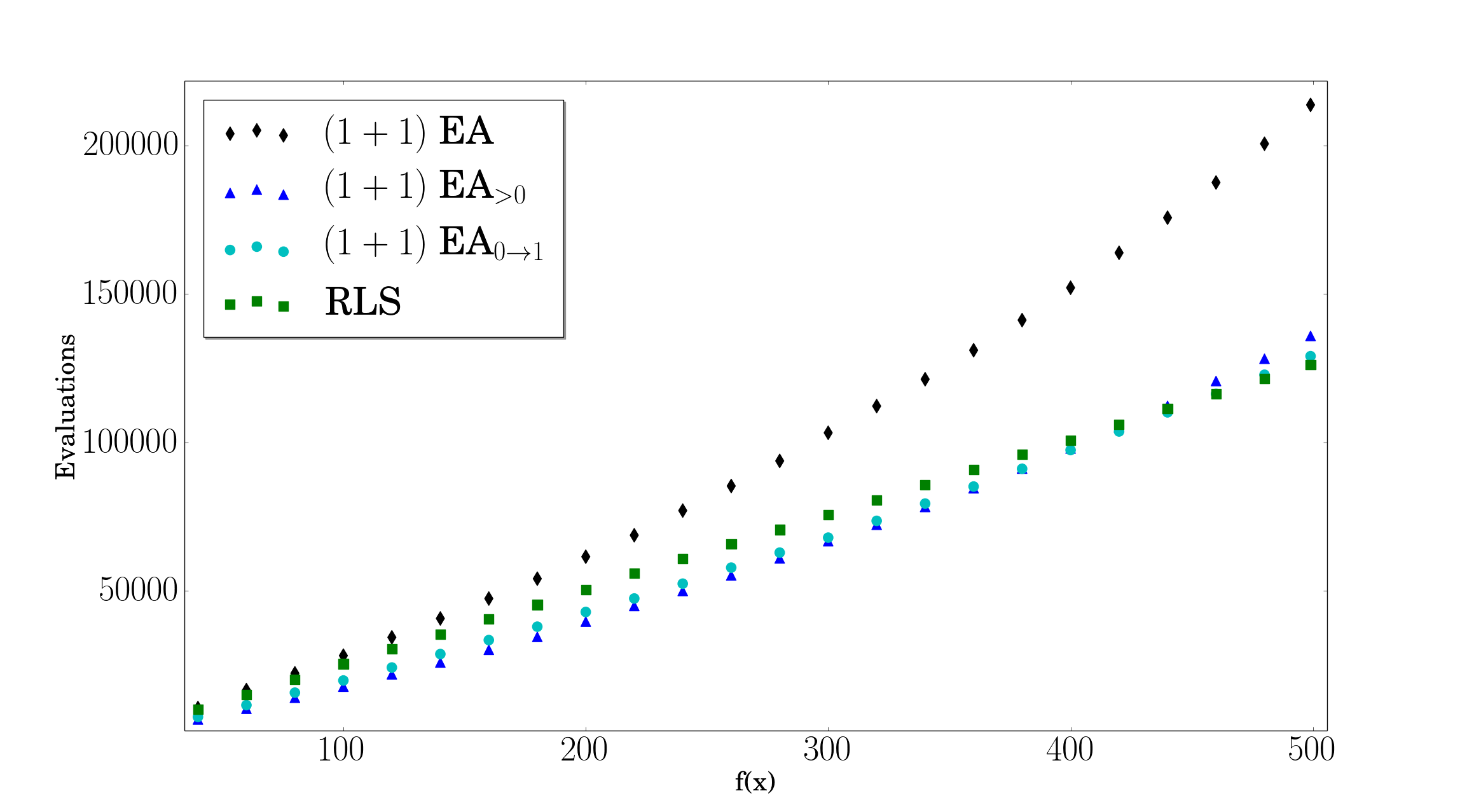}
\end{center}
\caption{Runtime profile for 100 runs of the respective algorithms (with mutation rate $p=1/n$ for the \oea and its variants) on \leadingones for problem size~$n=500$.}
\label{fig:LO-profile}
\end{figure} 
\section{Discussion and Future Works}
\label{sec:conclusions}

We hope to trigger with this work an extended discussion on how to make theoretical results in the domain of evolutionary computation more relevant and interpretable for practitioners. We have suggested two different steps into this direction, (1) do not charge an algorithm for function evaluations when the offspring equals one of its parents (in case this is easy to detect), and (2) report first hitting times not only for the optimum but also for intermediate fitness levels. Naturally, our work can only be a pointer to a more practice-aware theory, and we are aware that there are many more steps that have to be taken. In particular, we believe that the following observations need to be discussed in more detail. 
\begin{itemize}
	\item Many runtime statements report only expected optimization times. However, it is often interesting to understand the \emph{probability distribution of the optimization time}, in particular for problems where the variance can be large. Runtime analysis has recently seen an increased interest in these runtime distributions, cf., for example,~\cite{DoerrG13,Kotzing16} and follow-up works. 
	\item Similar to the previous point, problems exist where the expected optimization time can be very large even if the probability to hit an optimal solution within a small number of iterations is small. In~\cite{DoerrL15model} a so-called $p$-Monte Carlo complexity measure has been introduced, measuring the expected time to \emph{hit an optimal solution with probability at least $1-p$.} Similar suggestions can be found in~\cite{ZhouLLH12}.
	\item The runtime profiles suggested in Section~\ref{sec:profiles} complements the \emph{fixed-budget view} advocated by Jansen and Zarges~\cite{JansenZ14}. We feel that there is a need for combined measures that are capable of describing the \emph{anytime behavior of an EA}. Regret-based measure as used in machine learning could be a key here, but we haven't been able so far to identify a fully satisfying measure. 
	\end{itemize}
From a theoretical point of view, our suggested changes are easily implementable. Our work has nevertheless unveiled a quite remarkable result, the superiority of the \tpores over any unary unbiased black-box algorithm. We are confident that our performance measure will yield similar results for other problems and algorithms, with the potential of changing our view on fundamental questions like the benefits of crossover over mutation, (dis-)advantages of elitism vs. non-elitism, etc. 

\subsubsection*{Acknowledgments.}
We would like to thank Benjamin Doerr, Nikolaus Hansen, and Olivier Teytaud for several independent discussions around the topics of this work.

Our research benefited from the support of the ``FMJH Program Gaspard Monge in optimization and operation research'', and from the support to this program from EDF. 

Parts of our work have been inspired by COST Action CA15140: Improving Applicability of Nature-Inspired Optimisation by Joining Theory and Practice (ImAppNIO).

}
\newcommand{\etalchar}[1]{$^{#1}$}

\newpage
\appendix


\section{A Comment on the Benchmark Functions}
\label{app:benchmark}

Theoretical works are often criticized for regarding highly artificial benchmark problems. Indeed, while the power of EAs is certainly to be seen in applications to highly complex problems not admitting a thorough theoretical analysis, theoreticians regard simple benchmark functions like \onemax and \leadingones in the hope that, among other reasons,
\begin{itemize}
	\item they give insights into how the studied algorithms perform on the easier parts of a difficult optimization problem, 
	\item in order to understand some basic working principles of the algorithms, which can then be used for the analysis of more complex problems, more complex algorithms, for the development of new algorithmic ideas, etc., 
	\item the theoretical investigations, which even for seemingly simple algorithms and problems can be surprisingly complex, triggers the development of new analytical tools for the analysis of randomized algorithms, and
	\item for very precise mathematical statements a comparison between theoretical results and empirical performance can be made, helping us understand, for example, how the parameter choice influences the optimization time (and how the suggestions obtained via a thorough mathematical analysis differs from that obtained by empirical means). 
\end{itemize}

We furthermore note that even for \onemax and \leadingones, despite being studied since the very early days of theory of evolutionary computation, several unsolved problems exist, many of which are of seemingly simple nature such as the optimal dynamic mutation rate of the \oea for \onemax. 
Several important advances could be made in the last few years. These results have often required the development of rather sophisticated mathematical tools, as is witnessed by the different drift analysis theorems that have been developed in the last 7 years.

While we are certainly aware that the insights obtained from these benchmark problems (and the simplified evolutionary algorithms) may not always or not easily transfer to more realistic real-world optimization challenges, the concepts proposed in this work are applicable to a very broad range of theory- as well as practice-driven algorithms and problems. 

\subsection{Unbiasedness}
We would also like to point out that all the algorithms considered in this work are unbiased in the sense of Lehre and Witt~\cite{LehreW12}, i.e., their runtimes are identical for all functions that are obtained from the considered benchmark problems by composing them with a Hamming-distance preserving automorphism of the hypercube. We recall that a Hamming-automorphism of the hypercube is a one-to-one mapping $\sigma:\{0,1\}^n \rightarrow \{0,1\}^n$ such that for all $x,y \in \{0,1\}^n$ it holds that the Hamming distance $H(x,y)$ equals that of the images $H(\sigma(x), \sigma(y))$. The composition $f \circ \sigma$ of a pseudo-Boolean function $f:\{0,1\}^n \rightarrow \R$ with a Hamming-automorphism $\sigma$ has a fitness landscape that is isomorphic to that of~$f$ .

For \onemax this implies that all algorithms considered in this work behave identically on all functions of the form $\OM_z:\{0,1\}^n \rightarrow \R, x \mapsto |\{i \in [n] \mid x_i=z_i\}|$, where $z$ is an arbitrarily chosen binary string of length $n$. That is, all results reported for \onemax apply to any of these generalized functions $\OM_z$.

Similarly, for \leadingones, the composed functions are those of the form $\LO_{z,\pi}(x) := \max \{ i \in [0..n]\mid \forall j \leq i: z_{\pi(j)} = x_{\pi(j)}\}$, where $\pi$ is an arbitrary permutation of $[n]$ and $z$ an arbitrary binary string of length $n$. All results stated in this report for \LO applies to any of these generalized functions $\LO_{z,\pi}$.

\section{Tables with Statistical Data for the Experimental Results}
\label{app:tables}

\subsection{Statistical Details for Figure~\ref{fig:oeaOM}}
\label{app:tab:oeaOM}
{\footnotesize
\begin{longtable}{l|l||r|r|r|r|r|r|r}%
\hline
& &  \multicolumn{5}{c|}{\textbf{Percentile}}&&\textbf{StdDev/}\\
 \textbf{$n$} & \textbf{Algorithm} & 2& 25 & 50 & 75 & 98 &\textbf{Mean} & \textbf{Mean}\\
\hline
500 & \oea & 4885 & 6468 & 7436 & 8538 & 11566 & 7569 & 22.3\%\\
500 & \oeares & 3039 & 3919 & 4370 & 5125 & 7245 & 4642 & 21.8\%\\
500 & \oeashift & 2617 & 3430 & 3986 & 4443 & 5790 & 4045 & 19.9\%\\
500 & RLS & 2029 & 2493 & 2837 & 3491 & 4549 & 3050 & 23.4\%\\
\hline
1000 & \oea & 11803 & 14508 & 16137 & 17882 & 26448 & 16755 & 20.4\%\\
1000 & \oeares & 6855 & 9055 & 10067 & 11556 & 16313 & 10445 & 21.6\%\\
1000 & \oeashift & 6402 & 7416 & 8435 & 9459 & 13246 & 8697 & 18.8\%\\
1000 & RLS & 4922 & 6024 & 6887 & 7660 & 10031 & 7046 & 20.5\%\\
\hline
1500 & \oea & 20044 & 23308 & 26176 & 28633 & 43838 & 27129 & 20.5\%\\
1500 & \oeares & 11686 & 14839 & 16722 & 19588 & 26537 & 17358 & 19.3\%\\
1500 & \oeashift & 9412 & 11823 & 12998 & 14745 & 21422 & 13647 & 20.2\%\\
1500 & RLS & 7761 & 9761 & 10618 & 11662 & 14660 & 10781 & 15.4\%\\
\hline
2000 & \oea & 25377 & 32941 & 36046 & 42765 & 58925 & 38462 & 20.7\%\\
2000 & \oeares & 16741 & 21482 & 24132 & 25998 & 36370 & 24290 & 17.8\%\\
2000 & \oeashift & 13901 & 17286 & 19091 & 20668 & 26092 & 19302 & 15.5\%\\
2000 & RLS & 10721 & 13095 & 14798 & 17374 & 21276 & 15412 & 20.3\%\\
\hline
2500 & \oea & 34627 & 41715 & 47842 & 52869 & 67674 & 48362 & 16.6\%\\
2500 & \oeares & 21738 & 27207 & 29369 & 33089 & 41874 & 30495 & 17.2\%\\
2500 & \oeashift & 18977 & 22062 & 23831 & 27042 & 33713 & 24809 & 14.8\%\\
2500 & RLS & 14113 & 17054 & 18776 & 20921 & 26823 & 19192 & 15.2\%\\
\hline
3000 & \oea & 45746 & 51345 & 58094 & 65037 & 83549 & 59835 & 16.6\%\\
3000 & \oeares & 29319 & 32940 & 36178 & 39756 & 48863 & 37028 & 14.7\%\\
3000 & \oeashift & 22933 & 27044 & 30269 & 33849 & 47995 & 31344 & 18.9\%\\
3000 & RLS & 17045 & 21361 & 23586 & 26998 & 36154 & 24547 & 17.9\%\\
\hline
3500 & \oea & 53926 & 62548 & 68329 & 76639 & 97308 & 70608 & 15.6\%\\
3500 & \oeares & 32088 & 38664 & 42530 & 47916 & 63506 & 44874 & 18.4\%\\
3500 & \oeashift & 24333 & 33279 & 36474 & 39726 & 50482 & 36979 & 15.6\%\\
3500 & RLS & 22385 & 25164 & 27867 & 31124 & 38292 & 28471 & 14.3\%\\
\hline
4000 & \oea & 60303 & 73419 & 78445 & 87109 & 106910 & 80743 & 13.6\%\\
4000 & \oeares & 38121 & 46779 & 50841 & 56137 & 68302 & 52036 & 15.1\%\\
4000 & \oeashift & 32561 & 37462 & 41529 & 46539 & 59900 & 42703 & 16.1\%\\
4000 & RLS & 24922 & 29349 & 33493 & 36808 & 48134 & 33684 & 17.4\%\\
\hline
\label{tab:oeaOM}
\end{longtable}}

\subsection{Statistical Details for Figure~\ref{fig:oeaLO}}
\label{app:tab:oeaLO}
{\footnotesize
\begin{longtable}{l|l||r|r|r|r|r|r|r}%
\hline
& &  \multicolumn{5}{c|}{\textbf{Percentile}}&&\textbf{StdDev/}\\
 \textbf{$n$} & \textbf{Algorithm} & 2& 25 & 50 & 75 & 98 &\textbf{Mean} & \textbf{Mean}\\
\hline 
100 & \oea & 5431 & 7789 & 8637 & 9460 & 11877 & 8580 & 17.2\%\\
100 & \oeares & 3761 & 4831 & 5465 & 6104 & 8069 & 5574 & 18.5\%\\
100 & \oeashift & 3380 & 4434 & 5072 & 5615 & 7787 & 5194 & 19.1\%\\
100 & RLS & 3017 & 4391 & 4955 & 5565 & 6598 & 5005 & 16.9\%\\
\hline
200 & \oea & 25751 & 31088 & 34581 & 37314 & 41683 & 34400 & 12.1\%\\
200 & \oeares & 15687 & 19686 & 21026 & 23129 & 27120 & 21438 & 12.5\%\\
200 & \oeashift & 15589 & 18985 & 21164 & 22697 & 26539 & 20970 & 12.9\%\\
200 & RLS & 14954 & 18591 & 20134 & 21645 & 25353 & 20157 & 12.5\%\\
\hline
300 & \oea & 62296 & 71236 & 77600 & 83690 & 90711 & 77160 & 10.4\%\\
300 & \oeares & 39330 & 45502 & 48534 & 52776 & 59350 & 49080 & 9.9\%\\
300 & \oeashift & 35476 & 42809 & 46044 & 48796 & 57367 & 45961 & 10.8\%\\
300 & RLS & 36237 & 42464 & 45611 & 47555 & 52694 & 45283 & 9.1\%\\
\hline
400 & \oea & 109508 & 127281 & 135069 & 145748 & 162870 & 137192 & 10.4\%\\
400 & \oeares & 69665 & 79929 & 86964 & 90581 & 103756 & 86343 & 9.0\%\\
400 & \oeashift & 68047 & 76346 & 81149 & 86265 & 97555 & 81772 & 8.4\%\\
400 & RLS & 67930 & 75865 & 79498 & 83673 & 93974 & 80153 & 7.6\%\\
\hline
500 & \oea & 180855 & 203457 & 212353 & 224810 & 243611 & 213616 & 7.1\%\\
500 & \oeares & 118828 & 128745 & 134675 & 143121 & 156814 & 135853 & 7.1\%\\
500 & \oeashift & 110463 & 121950 & 128466 & 134415 & 148957 & 129039 & 7.1\%\\
500 & RLS & 100044 & 119000 & 126519 & 132246 & 151616 & 126160 & 8.7\%\\
\hline
600 & \oea & 258824 & 294309 & 304544 & 324434 & 356231 & 308007 & 7.4\%\\
600 & \oeares & 166032 & 186940 & 197054 & 207332 & 220871 & 196690 & 7.1\%\\
600 & \oeashift & 159298 & 180316 & 187209 & 194224 & 212286 & 187251 & 6.8\%\\
600 & RLS & 159750 & 173582 & 180506 & 185784 & 203100 & 180911 & 5.7\%\\
\hline
\label{tab:oeaLO}
\end{longtable}}

\subsection{Statistical Details for Figure~\ref{fig:ga}}
\label{app:tab:ga}

{\footnotesize
\begin{longtable}{l|l||r|r|r|r|r|r|r}%
\hline
& &  \multicolumn{5}{c|}{\textbf{Percentile}}&&\textbf{StdDev/}\\
 \textbf{$n$} & \textbf{Algorithm} & 2& 25 & 50 & 75 & 98 &\textbf{Mean} & \textbf{Mean}\\
\hline 
500 & \ga & 4082 & 4532 & 4746 & 4986 & 5748 & 4791 & 8.8\%\\
500 & RLS & 2029 & 2493 & 2837 & 3491 & 4549 & 3050 & 23.4\%\\
500 & \tpores & 1951 & 2422 & 2816 & 3343 & 4317 & 2928 & 21.4\%\\
500 & \gaopt & 2752 & 3065 & 3280 & 3499 & 3929 & 3300 & 9.4\%\\
\hline
1000 & \ga & 8238 & 9206 & 9684 & 10222 & 11120 & 9754 & 8.1\%\\
1000 & RLS & 4922 & 6024 & 6887 & 7660 & 10031 & 7046 & 20.5\%\\
1000 & \tpores & 4252 & 5503 & 6090 & 6700 & 8753 & 6200 & 17.0\%\\
1000 & \gaopt & 5855 & 6378 & 6716 & 6993 & 8060 & 6771 & 8.3\%\\
\hline
1500 & \ga & 13134 & 14162 & 14604 & 15234 & 16816 & 14767 & 6.2\%\\
1500 & RLS & 7761 & 9761 & 10618 & 11662 & 14660 & 10781 & 15.4\%\\
1500 & \tpores & 7219 & 8872 & 9554 & 10422 & 11911 & 9642 & 12.0\%\\
1500 & \gaopt & 8985 & 9522 & 10058 & 10446 & 11742 & 10054 & 6.9\%\\
\hline
2000 & \ga & 17502 & 18960 & 19604 & 20384 & 22240 & 19749 & 5.7\%\\
2000 & RLS & 10721 & 13095 & 14798 & 17374 & 21276 & 15412 & 20.3\%\\
2000 & \tpores & 10791 & 12179 & 13393 & 14963 & 20410 & 13856 & 16.4\%\\
2000 & \gaopt & 12146 & 13139 & 13511 & 13940 & 15477 & 13638 & 6.4\%\\
\hline
2500 & \ga & 21828 & 24024 & 24558 & 25234 & 26874 & 24614 & 4.2\%\\
2500 & RLS & 14113 & 17054 & 18776 & 20921 & 26823 & 19192 & 15.2\%\\
2500 & \tpores & 13511 & 16081 & 17212 & 18755 & 23870 & 17703 & 13.7\%\\
2500 & \gaopt & 15289 & 16549 & 17076 & 17673 & 19275 & 17134 & 5.8\%\\
\hline
3000 & \ga & 27300 & 28716 & 29596 & 30430 & 32788 & 29736 & 4.9\%\\
3000 & RLS & 17045 & 21361 & 23586 & 26998 & 36154 & 24547 & 17.9\%\\
3000 & \tpores & 17493 & 19601 & 21447 & 23292 & 30849 & 22081 & 15.4\%\\
3000 & \gaopt & 18860 & 19906 & 20549 & 21108 & 23088 & 20641 & 5.1\%\\
\hline
3500 & \ga & 31888 & 33516 & 34624 & 35264 & 37190 & 34544 & 3.8\%\\
3500 & RLS & 22385 & 25164 & 27867 & 31124 & 38292 & 28471 & 14.3\%\\
3500 & \tpores & 19598 & 22805 & 25340 & 27811 & 36918 & 25925 & 15.6\%\\
3500 & \gaopt & 21858 & 23468 & 24119 & 24933 & 27131 & 24276 & 5.0\%\\
\hline
4000 & \ga & 36648 & 38758 & 39848 & 40390 & 43014 & 39733 & 3.8\%\\
4000 & RLS & 24922 & 29349 & 33493 & 36808 & 48134 & 33684 & 17.4\%\\
4000 & \tpores & 23175 & 26752 & 28673 & 31389 & 36222 & 29372 & 12.3\%\\
4000 & \gaopt & 25106 & 26564 & 27466 & 28139 & 30187 & 27496 & 4.7\%\\
\hline
4500 & \ga & 41698 & 43520 & 44434 & 45298 & 48550 & 44664 & 3.9\%\\
4500 & RLS & 27699 & 32726 & 35429 & 39177 & 49298 & 36439 & 13.5\%\\
4500 & \tpores & 27140 & 30578 & 32795 & 35492 & 43957 & 33682 & 12.7\%\\
4500 & \gaopt & 28326 & 29927 & 30900 & 31551 & 33991 & 30988 & 4.8\%\\
\hline
5000 & \ga & 46568 & 48378 & 49468 & 50912 & 54402 & 49857 & 4.9\%\\
5000 & RLS & 30940 & 37046 & 40483 & 46237 & 55083 & 41555 & 15.3\%\\
5000 & \tpores & 30441 & 34279 & 38186 & 40793 & 50552 & 38437 & 13.3\%\\
5000 & \gaopt & 32234 & 33514 & 34348 & 35436 & 38117 & 34666 & 4.4\%\\
\hline
\label{tab:ga}
\end{longtable}}

\subsection{Statistical Details for Figure~\ref{fig:OMprofile}}
\label{app:tab:OMprofile}
{\footnotesize
\begin{longtable}{l|l||r|r|r|r|r|r|r}%
\hline
& &  \multicolumn{5}{c|}{\textbf{Percentile}}&&\textbf{StdDev/}\\
 \textbf{$n$} & \textbf{Algorithm} & 2& 25 & 50 & 75 & 98 &\textbf{Mean} & \textbf{Mean}\\
\hline 
2600 & \oea & 2034 & 2213 & 2309 & 2385 & 2576 & 2304 & 5.6\%\\
2600 & \oeares & 1257 & 1394 & 1454 & 1524 & 1617 & 1459 & 6.0\%\\
2600 & \oeashift & 1287 & 1390 & 1436 & 1475 & 1599 & 1435 & 4.7\%\\
2600 & RLS & 1230 & 1376 & 1444 & 1480 & 1588 & 1428 & 5.7\%\\
\hline
2800 & \oea & 3067 & 3300 & 3387 & 3511 & 3666 & 3401 & 4.3\%\\
2800 & \oeares & 1891 & 2087 & 2157 & 2230 & 2303 & 2156 & 4.5\%\\
2800 & \oeashift & 1953 & 2046 & 2092 & 2138 & 2250 & 2091 & 3.7\%\\
2800 & RLS & 1836 & 1973 & 2050 & 2118 & 2205 & 2050 & 4.3\%\\
\hline
3000 & \oea & 4357 & 4659 & 4770 & 4901 & 5063 & 4768 & 3.6\%\\
3000 & \oeares & 2741 & 2945 & 3033 & 3084 & 3234 & 3018 & 3.8\%\\
3000 & \oeashift & 2721 & 2848 & 2900 & 2966 & 3118 & 2905 & 3.5\%\\
3000 & RLS & 2595 & 2701 & 2765 & 2838 & 3015 & 2776 & 3.7\%\\
\hline 
3200 & \oea & 6126 & 6414 & 6570 & 6714 & 6955 & 6567 & 3.3\%\\
3200 & \oeares & 3849 & 4064 & 4148 & 4249 & 4425 & 4153 & 3.2\%\\
3200 & \oeashift & 3646 & 3840 & 3914 & 3992 & 4216 & 3924 & 3.1\%\\
3200 & RLS & 3457 & 3578 & 3656 & 3764 & 3899 & 3671 & 3.2\%\\
\hline
3400 & \oea & 8353 & 8849 & 9027 & 9209 & 9605 & 9037 & 3.1\%\\
3400 & \oeares & 5350 & 5605 & 5696 & 5820 & 6086 & 5705 & 3.1\%\\
3400 & \oeashift & 4898 & 5206 & 5288 & 5376 & 5600 & 5290 & 2.9\%\\
3400 & RLS & 4486 & 4729 & 4812 & 4910 & 5130 & 4817 & 3.0\%\\
\hline
3600 & \oea & 11821 & 12483 & 12720 & 13008 & 13392 & 12730 & 3.0\%\\
3600 & \oeares & 7470 & 7932 & 8049 & 8167 & 8644 & 8050 & 3.1\%\\
3600 & \oeashift & 6869 & 7162 & 7282 & 7399 & 7739 & 7286 & 2.7\%\\
3600 & RLS & 6055 & 6297 & 6409 & 6536 & 6814 & 6428 & 3.0\%\\
\hline
3800 & \oea & 18255 & 19041 & 19547 & 19965 & 20600 & 19505 & 3.2\%\\
3800 & \oeares & 11636 & 12109 & 12362 & 12624 & 13119 & 12362 & 3.1\%\\
3800 & \oeashift & 10141 & 10611 & 10834 & 11000 & 11396 & 10813 & 2.8\%\\
3800 & RLS & 8706 & 9016 & 9202 & 9351 & 9742 & 9199 & 2.8\%\\
\hline
4000 & \oea & 60303 & 73419 & 78445 & 87109 & 106910 & 80743 & 13.6\%\\
4000 & \oeares & 38121 & 46779 & 50841 & 56137 & 68302 & 52036 & 15.1\%\\
4000 & \oeashift & 32561 & 37462 & 41529 & 46539 & 59900 & 42703 & 16.1\%\\
4000 & RLS & 24922 & 29349 & 33493 & 36808 & 48134 & 33684 & 17.4\%\\
\hline
\label{tab:OMprofile}
\end{longtable}}

\subsection{Statistical Details for Figure~\ref{fig:LO-profile}}
\label{app:tab:LOprofile}
{\footnotesize
\begin{longtable}{l|l||r|r|r|r|r|r|r}%
\hline
& &  \multicolumn{5}{c|}{\textbf{Percentile}}&&\textbf{StdDev/}\\
 \textbf{$n$} & \textbf{Algorithm} & 2& 25 & 50 & 75 & 98 &\textbf{Mean} & \textbf{Mean}\\
\hline
200 & RLS & 36642 & 45223 & 49595 & 54415 & 63328 & 50189 & 13.1\% \\
200 & \oea & 43978 & 55993 & 60523 & 66938 & 77029 & 61317 & 12.7\% \\
200 & \oeares & 26407 & 36169 & 39296 & 42197 & 49334 & 39417 & 12.8\% \\
200 & \oeashift & 32697 & 38409 & 42635 & 45902 & 50956 & 42640 & 11.6\% \\
\hline 
300 & RLS & 58189 & 68856 & 75373 & 80530 & 93329 & 75406 & 11.7\% \\
300 & \oea & 82035 & 95247 & 102474 & 109291 & 124700 & 102831 & 9.4\% \\
300 & \oeares & 53246 & 61699 & 66135 & 69922 & 80272 & 66381 & 9.5\% \\
300 & \oeashift & 55803 & 64063 & 68162 & 71373 & 80814 & 67726 & 8.9\% \\
\hline 
400 & RLS & 80955 & 94774 & 99848 & 106075 & 122599 & 100505 & 9.6\% \\
400 & \oea & 124398 & 141719 & 150333 & 161220 & 176583 & 151502 & 8.5\% \\
400 & \oeares & 80748 & 91999 & 98433 & 102542 & 112088 & 97653 & 8.0\% \\
400 & \oeashift & 81130 & 91151 & 97831 & 101446 & 111312 & 97106 & 7.5\% \\
\hline 
500 & RLS & 100044 & 119000 & 126519 & 132246 & 151616 & 126160 & 8.7\% \\
500 & \oea & 180855 & 203457 & 212353 & 224810 & 243611 & 213616 & 7.1\% \\
500 & \oeares & 118828 & 128745 & 134675 & 143121 & 156814 & 135853 & 7.1\% \\
500 & \oeashift & 110463 & 121950 & 128466 & 134415 & 148957 & 129039 & 7.1\% 
\label{tab:LOprofile}
\end{longtable}}
\end{document}